\documentclass[pdflatex,sn-mathphys,iicol]{sn-jnl}%

\usepackage[utf8]{inputenc}
\usepackage[T1]{fontenc}

\jyear{2023}
\usepackage{natbib}
\usepackage{listings}

\usepackage{graphicx}
\usepackage[utf8]{inputenc}
\usepackage[algo2e,ruled,linesnumbered]{algorithm2e}
\usepackage{hyperref}       %
\usepackage{url}            %
\usepackage{amssymb}%
\usepackage{amsmath}

\usepackage{multirow}
\usepackage{xcolor}
\usepackage{array}
\usepackage{enumerate}
\usepackage{stmaryrd}%
\usepackage{ifthen}
\usepackage{dsfont}
\usepackage{subfigure}
\usepackage{comment}
\usepackage[normalem]{ulem}

\usepackage{soul}
\theoremstyle{thmstyleone}%
\newtheorem{theorem}{Theorem}%
\newtheorem{defn}[theorem]{Definition}%
\newtheorem{thm}[theorem]{Theorem}%
\newtheorem{proper}[theorem]{Property}%
\newtheorem{lemma}[theorem]{Lemma}%
\let\footnote\thanks

\definecolor{color1}{HTML}{D73027}
\definecolor{color2}{HTML}{FC8D59}
\definecolor{color3}{HTML}{FEE090}
\definecolor{color4}{HTML}{E0F3F8}
\definecolor{color5}{HTML}{91BFDB}
\definecolor{color6}{HTML}{4575B4}

\definecolor{rouge}{rgb}{1.0,0.0,0.0}
\definecolor{blue}{rgb}{0.0,0.0,1.0}

\newcommand*{\ie}{\textit{i.e.},}

\newcommand*\FinalRevision{false}
\newcommand*{\REMM}[1]{\ifthenelse{\equal{\FinalRevision}{false}}{\fbox{\begin{minipage}{0.4\textwidth}\textsl{\textcolor{rouge}{#1}}\end{minipage}}}{#1}}
\newcommand*{\REM}[1]{\ifthenelse{\equal{\FinalRevision}{false}}{\textbf{\uline{\textcolor{rouge}{#1}}}}{}}
\newcommand*{\TODO}[1]{\ifthenelse{\equal{\FinalRevision}{false}}{\fbox{\begin{minipage}{0.4\textwidth}\textsl{\textcolor{blue}{#1}}\end{minipage}}}{#1}}
\newcommand*{\rev}[1]{\ifthenelse{\equal{\FinalRevision}{false}}{\textbf{\textcolor{blue}{#1}}}{}}
\newcolumntype{M}[1]{>{\centering}m{#1}}

\newcommand*{\Cite}[1]{~\cite{#1}}

\newcommand*{\mset}[1]{\left\{#1\right\}}

\newcommand*{\cardinal}[1]{	\lvert#1\rvert}

\newcommand{\suchthat}[0]{\; \vert \;}

\newcommand{\Xfig}[1]{Figure~\ref{#1}}

\usepackage{enumitem}
\setlist{topsep=1.0em, itemsep=1.0em} 

\begin{document}
\title{Playing with Kruskal: a state-of-the-art report on watershed cuts}
\author[1]{Jean Cousty}\email{jean.cousty@esiee.fr}
\author*[3,1]{Laurent Najman}\email{laurent.najman@esiee.fr}
\author[1]{Benjamin Perret}\email{benjamin.perret@esiee.fr}
\author[2]{Deise Santana Maia}\email{deise.santanamaia@univ-lille.fr}

\affil[1]{\orgname{Univ Gustave Eiffel, CNRS}, \orgdiv{LIGM}, \orgaddress{\postcode{F-77454} \city{Marne-la-Vall\'ee},  \country{France}}}
\affil[2]{\orgname{Univ Lille, CNRS}, \orgdiv{UMR 9189 CRIStAL}, \orgaddress{\postcode{F-59000} \city{Lille},  \country{France}}}
\affil[3]{\orgdiv{Math. Department}, \orgname{Khalifa University}, \orgaddress{\city{Abu Dhabi}, \country{UAE}}}

\abstract{
In the framework of edge-weighted graphs, watersheds have proven to be linked to well-known optimization problems, as Minimum Spanning Tree, which allowed the design of efficient algorithms for computing (hierarchical) watershed segmentations.
In the present article, after reviewing the literature related to watershed segmentation, we present a detailed end-to-end pipeline of algorithms to compute (hierarchical) watershed segmentations, starting from the computation of graph-based image representations, up to the computation of connected components of the final (hierarchical) segmentation. 
We consider the several variations of watersheds, including their supervised and unsupervised versions, and the various ways of computing seeds, to name a few. 
For the first time, we bring together all these watershed notions and algorithms in a compact and understandable way. We aim at providing a reference for those interested in employing and reimplementing the watershed segmentation framework for their task at hand. %
}


\maketitle

\section{Introduction}

The watershed transform was first proposed as a powerful tool for the segmentation of grayscale digital images \cite{digabel-lantuejoul_1978,beucher-lantuejoul_iwip1979}. Since then, numerous definitions and algorithms have been developed to implement the watershed transform in both 2D and 3D image domains, as well as within the framework of graphs \cite{vincent-soille_pami1991,meyer_afcet1991,najman-schmitt_sigproc1994,meyer_sigproc1994,couprie-bertrand_spie97,roerdink_fi2000,cousty-et-al_pami2008}. In all cases, the underlying idea of the watershed transform is to interpret an image (or a weighted graph) as a topographic surface, where a set of connected pixels with equal gray values surrounded by pixels of strictly higher gray values (or a set of adjacent vertices/edges surrounded by vertices/edges of strictly greater weights) forms a regional minimum of the surface. Each regional minimum corresponds to a zone of influence, known as a catchment basin. From any point (pixel or vertex) within the catchment basin of a regional minimum, there exists a descending path to that minimum. A descending path is defined as either a sequence of connected pixels with non-increasing gray levels or a sequence of vertices connected by edges with non-increasing weights.

A watershed segmentation is, therefore, a partition of an image into its catchment basins. 
Its hierarchical version consists of a sequence of nested partitions obtained by progressively merging the regions of an initial watershed segmentation \cite{beucher_ismm1994,najman-schmitt_pami1996,maia-et-al_jmiv2020}. 
When formalized within the framework of edge-weighted graphs, (hierarchical) watersheds are closely related to the optimization problem of finding minimum spanning trees (MST), as proved in \cite{cousty-et-al_pami2008,cousty-najman_ismm2011,cousty-et-al_ismm2013}.

In this article, we present an overview of state-of-the-art watershed definitions for edge-weighted graphs, covering flat, hierarchical, supervised, and unsupervised approaches. Then, following the trend of Najman et al. \cite{najman-et-al_ismm2013}, we revisit the watershed paradigm through the lens of Kruskal’s algorithm for computing MSTs \cite{kruskal-ams1956}. By doing so, we reveal a simple, modular, and powerful framework for computing watershed segmentations—flat or hierarchical, supervised or unsupervised—using a single underlying data structure: the Binary Partition Tree by Altitude Ordering (BPTAO) \cite{cousty-et-al_ismm2013}. This tree encodes the structure of the MST, supports efficient seed propagation, and enables the direct extraction of segmentation boundaries or hierarchical saliency maps \cite{cousty-et-al_jmiv2018}. By systematically organizing and extending previous results in the literature, we describe an end-to-end pipeline that goes from an input image to the final watershed segmentation, passing through graph construction, tree computation, seed selection, and cut extraction.

Our goal is to provide a comprehensive and versatile algorithmic state-of-the-art reference for researchers and practitioners working with watershed-based segmentation. We emphasize the interplay between combinatorial optimality, morphological filtering, and hierarchy construction, and demonstrate that all major variants of watershed cuts—basic, seeded, and hierarchical—can be derived and computed within our Kruskal-inspired framework. In doing so, we hope to foster better understanding, facilitate implementation, and open new avenues for integrating watershed cuts with modern learning-based methods.

Section~\ref{sec:relatedWork} briefly  reviews watershed approaches for image segmentation. Rather than aiming at an exhaustive survey, we position watershed cuts within a broader context, emphasizing their connections to combinatorial optimization and computational geometry. Section~\ref{sec:watershedCutsHierarchies} revisits the theoretical foundations of the watershed-cut framework. Building upon these principles, Section~\ref{sec:algoOverview} introduces the generic Playing-with-Kruskal algorithmic scheme, designed to unify the computation of the various instances of watershed cuts. The construction of the central data structure of this framework, namely the Binary Partition Tree, is detailed in Section~\ref{sec:bpt}. Section~\ref{sec:selection_seeds} addresses seed computation, while Sections~\ref{sec:extraction_watershed_cut} and~\ref{sec:extraction_cc} explain how  watershed-cut segmentations are obtained from the binary partition tree and the selected seeds. Finally, Section~\ref{sec:useCases} presents illustrative use cases implemented using the Higra library \cite{PCCGKN:softwarex2019}, demonstrating the practical relevance and flexibility of the proposed framework.

\section{Related work}
\label{sec:relatedWork}

Since the first definition of watershed segmentation in the late 1970s~\cite{digabel-lantuejoul_1978}, several algorithms and paradigms have been proposed to compute watershed segmentations, notably in the work of Vincent and Soille~\cite{vincent-soille_pami1991}. Oversegmentation being a common aspect of the original watershed definitions, where markers are defined as the set of regional minima, various approaches have been proposed to alleviate this issue, either by filtering the initial regions, computing a hierarchical segmentation, improving the selection of markers, or modifying the computation of the image gradient, such as the techniques proposed in \cite{beucher_ismm1994,najman-schmitt_pami1996,levner2007classification,angulo2007stochastic,wolf2018mutex}. Based on constant developments over the years, watershed segmentation attracted a lot of attention in many applicative fields such as remote sensing \cite{tarabalka-et-al_PR2010}, medical imaging \cite{grau-et-al_tmi2004,cousty-et-al_IVC2010}, biological imaging \cite{lin-et-al_cyt2003}, and material science \cite{faessel-jeulin_mic2010, fotos-et-al_jms2023}. In this section, we review the main steps of watershed developments for image analysis and discuss its possible or established connections to other important research fields, mainly combinatorial optimization and geometry.

\subsection{Early flooding-based methods on vertex-weighted graphs}

In the early work of Vincent and Soille \cite{vincent-soille_pami1991} and of Meyer \cite{meyer_afcet1991}, the proposed flooding algorithm relies on vertex-weighted graphs whose vertices correspond to image pixels and whose edges are given by an adjacency relation on the pixels. It starts with a set of marked seed pixels. Each seed pixel is given a unique label. Then, the marked pixels are iteratively augmented by incorporating neighboring pixels that are not yet marked, following an increasing order of weights. The final segmentation is a pixel labeling that indicates the catchment basins of the pixels and possibly the label ``watershed divide" for some pixels. Other flooding-based approaches are also described in \cite{soille1990determining,meyer1999multiscale,sofou2008generalized,meyer2019topographical}.

\subsection{Hierarchical and stochastic watersheds}
Hierarchical watershed approaches \cite{beucher_ismm1994, najman-schmitt_ismm94, najman-schmitt_pami1996, meyer_ismm1996, cousty-najman_ismm2011, cousty-et-al_ismm2013, maia-et-al_jmiv2020}  were introduced to extend the classical watershed to multiple scales. By organizing the image segmentation into nested regions forming a hierarchy \cite{benzecri-et-al_dunod1973,leclerc_1981,guigues-et-al_ijcv2006,cousty-et-al_jmiv2018}, it becomes possible to analyze structures at different levels of detail. One widely used approach relies on geodesic saliency (also known as ultrametric contour map \cite{arbelaez_cvprw2006,amfm_pami2011}), which ranks region contours based on their persistence across scales. This allows the extraction of a multi-scale hierarchy of regions that reduces spurious basins while preserving relevant structures.

Angulo and Jeulin \cite{angulo2007stochastic} proposed the stochastic watershed segmentation, in which the final segmentation is preceded by the computation of a probability density function (PDF) of contours. These contours are obtained by combining watershed segmentations computed from different random sets of markers. The idea is that important contours persist across multiple random marker sets, increasing their value in the probability density function, while small and low-contrast contours fade, leading to reduced oversegmentation. An exact calculation of the PDF is proposed by Stawiaski  and Meyer~\cite{stawiaski2010stochastic}, while an efficient algorithm for its computation is proposed by Malmberg and Hendriks~\citep{malmberg2014efficient}. In \cite{bernander2013improving}, an improvement to the original method introduces noise at every iteration and uses a randomly placed grid to distribute the markers.

\subsection{Edge-weighted graphs and combinatorial optimization}
Since the early 2000s, edge-weighted graphs have established themselves as a powerful framework for image segmentation. As presented in the following subsection, it became the reference framework for watershed segmentation allowing establishing connections between watersheds and important combinatorial optimization problems. The  most notable seminal work of Felzenszwalb and Huttenlocher \cite{felzenszwalb2004efficient} proposes a nearly linear-time segmentation algorithm in which an input image is represented as a graph (pixels as nodes and edge weights encoding dissimilarity between pixels) whose clustering strategy is closely related to the Minimum Spanning Tree (MST) optimization problem. Since then, this simple and efficient algorithm has been extended and adapted to a wide range of segmentation frameworks, such as~\cite{cayllahua2020hierarchical}. In the framework of edge-weighted graphs, Boykov et al. \cite{boykov-et-al_pami2002} consider the min-cut/max-flow optimization problem to solve various image analysis problems including segmentation. In \cite{falcao2004image}, the authors propose the Image Foresting Transform (IFT), which implements an extended version of Dijkstra's shortest paths forests algorithm on edge-weighted graphs, where each vertex is clustered with the initial seed (or subgraph) to which its path is shortest among all initial seeds. 

\subsection{Watershed cuts and unifying frameworks}

Following this trend, watershed cuts were introduced in the framework of edge-weighted graphs in the work of \cite{cousty-et-al_ismm2007,cousty-et-al_pami2008,cousty-et-al_pami2009}. They are defined following the intuitive idea of drops of water flowing on a topographic surface. In contrast to watersheds in previously considered discrete settings, the consistency of watershed cuts has been established: they can be defined equivalently by their ``catchment basins" through the steepest descent property or by the ``dividing lines" that separate these basins through the drop of water principle. Watershed cuts are also closely related to MSTs through an equivalence theorem, and hence to Kruskal's MST algorithm, which, as detailed in the later sections, provides a key ingredient for efficient and versatile computations for the regular flat case, the marker-based versions, and the hierarchical extensions.

Additionally, this strong connection to MSTs opened the doors to establishing links between watershed cuts and other optimization approaches on graphs \cite{allene-et-al_ismm2007}. Especially, it is shown that min-cuts coincide with watersheds for some particular weight functions. These links were further studied in \cite{couprie-et-al_pami2010, najman2017extending} that introduce a new generic combinatorial optimization problem on graphs parameterized by two exponent values $p$ and $q$ and allows several known methods to be unified. Depending on the values of $p$ and $q$, the solution is a minimal cut, a random walker~\cite{grady2006random}, a shortest paths forest, or finally, when $p$ tends to infinity, a new combinatorial object called a power watershed. The power watershed is a watershed cut for which ties are discriminated by a secondary optimality criterion. This result links several methods that are useful in image analysis and combines their advantages.

\subsection{Extension, learning and large-scale processing}

More recently, Mutex Watersheds \cite{wolf2018mutex}, which can be seen as an application of the power-watershed framework, extended the multicut formulation to a “seed-free” watershed approach with time complexity comparable to MSTs. The method incorporates both attractive interactions, which encourage merging of similar pixels, and repulsive long-range edges, allowing the graph to account for both short and long-range relationships between pixels. Its main advantages are that it does not require seeds or a predefined number of regions.

In the current deep learning era, watershed segmentation remains an important component of many image analysis pipelines and is still widely used in conjunction with modern machine learning methods, often as a pre- or post-processing step (see, e.g., \cite{machairas2015waterpixels,amfm_pami2011,eschweiler2019cnn,wolf2017learned,lux2019dic,maia2022watershed}).

In \cite{challa2019watersheds}, the authors proposed, in the spirit of random forest~\cite{breiman2001random}, to ensemble many ``weak" watershed cuts, each one using only a random subset of the features for computing the edge-weights. In combination with a triplet loss, this has been proved particularly efficient for segmenting hyperspectral images~\cite{Tripletws}.

Other approaches include interactive, out-of-core, and deep learning based segmentation approaches \cite{lebon2023interactive,lefevre2024out,barbosa2025hierarchy}. In \cite{lebon2023interactive}, the authors introduce an efficient interactive implementation where seeds can be added or removed without recomputing the entire segmentation, achieving a speed-up of up to 95 times, with some  multi-core parallelizations, when compared to the previous state-of-the-art high-performance library. In \cite{lefevre2024out}, an out-of-core watersheds algorithm is proposed for data which cannot fit the main memory, enabling parallelized computation without compromising the correctness of the watershed-cut framework. Da Fonseca et al.~\cite{barbosa2025hierarchy} leverages a deep network to provide user-based markers optimized for connected, smooth regions, while Lapertot et al.~\cite{lapertot2024end} uses deep learning to generate hierarchical contours, producing semantically coherent segmentation levels.

\subsection{Beyond images: surfaces, geometry, and terrains}

The notion of a watershed has also been applied to the partitioning of 3D surface meshes, notably following the  work of Mangan and Whitaker~\cite{mangan-whitaker_tvcg1999}. The watershed cut approach is, for instance, applied  in~\cite{Philipp-et-al_pr2011} to tasks such as artwork 3D model segmentation prior to classification and indexation. 

To partition surfaces, notions closely related to those involved in watershed-based image segmentation have been developed in the computational geometry community, in particular through ascending, descending, and Morse-Smale complexes \cite{edelsbrunner-et-al_CG2003,edelsbrunner-et-al_DCG2003,edelsbrunner-harer_2009}, defined in the setting  of simplicial complexes and often relying on discrete Morse functions~\cite{delgado2014skeletonization}. Each ascending complex is associated with one minimum and comprises the locations which can be reached from this minimum by an ascending line. From this perspective, ascending complexes can be thought of as analogous to the watershed regions (the catchment basins). The descending complex is a dual notion. It can be interpreted as the regions delineated by the thalwegs, a thalweg (the dual of a watershed) being made of ``closed contours" that pass through the bottoms of the valleys and separating the regional maxima. The Morse-Smale complex is obtained by intersecting the ascending and descending complexes. It is tightly linked to the critical points of the altitude function, namely minima, maxima and saddles. 

For Morse-Smale functions on 2D manifold surfaces, the crossing of ascending and descending complexes is necessarily at a saddle. The strong relation between watersheds and saddles (also referred to as pass points) is well established in mathematical morphology notably within the topological watershed framework \cite{bertrand_jmiv2005}.  The notion of watershed from image analysis has further been studied in \cite{cousty2014collapses} on simplicial complexes pseudomanifolds and, later,  for discrete Morse functions defined on such spaces \cite{bertrand-et-al_jmiv2023}. Conversely, the watershed has inspired approaches based on discrete Morse theory~\cite{delgado2014skeletonization}.

Finally, it is worth noting that approaches from both image analysis and  computational geometry to partition surfaces are widely applied to digital elevation models in hydrography and geoscience \cite{wu-et-al_geomorph2019, guilbert_jsis2021,soille-et-al_wrr2003,comic-et-al_SIT2005, hiatt-et-al_jgr2020,kleinhans-et-al_jcg2019,ophelders-et-al_socg2025}. In this context,  the term ridge is also frequently used  to denote watershed-like structures.

\subsection{Implementations}
The literature review presented in this section is not intended to be exhaustive, but rather to provide an overview of a selection of methods derived from the notion of watersheds and related concepts.  Watershed algorithms  for 2D and sometimes 3D image processing are widely implemented in both open-source libraries, such as scikit-image, ITK, and OpenCV, and commercial or domain-specific platforms including MATLAB, ImageJ, and Amira/Avizo. For a broad review on open-source watershed algorithm implementations, we refer the reader to \cite{kornilov2018overview,kornilov2022review} and to \cite{PCCGKN:softwarex2019}, the latter being the hierarchical graph analysis library used for the illustrations in this article.

\color{black}

\section{Watershed cuts and hierarchies}
\label{sec:watershedCutsHierarchies}

In this section, we provide an overview of the watershed cuts family. More precisely, we present three variations of watershed cuts:
\begin{itemize}
    \item basic watershed cuts: the method is fully unsupervised, and the input function is automatically partitioned into regions associated to the minima of the input function;
    \item seeded watershed cuts: the method is semi-supervised and the user, or an automated procedure, must additionally provide markers located inside each region which is to be segmented. As a result, the algorithm partitions the images into regions associated to the markers.
    \item hierarchical watershed cuts: the method is either fully unsupervised or semi-supervised and considers a hierarchical image representation where the watershed regions of the finest scale (associated to all minima of the input function) are progressively merged to form watershed regions of higher scales representing object of higher semantic level.
\end{itemize}

\subsection{Watershed cuts}
\label{sec:wsc}
The watershed transform introduced by Digabel and Lantuejoul \cite{digabel-lantuejoul_1978} for morphological segmentation and later popularized by Vincent and Soille \cite{vincent-soille_pami1991} is used as a fundamental step in many image segmentation procedures. A grayscale image, or more generally a function, is seen as a topographic surface: the gray values become the elevations, the basins, and valleys correspond to dark areas whereas the mountains and crest lines correspond to light areas. Intuitively, the watershed is a subset of the domain, located on the ridges of the topographic surface, that delineates its catchment basins. It may be thought of as a separating line-set from which a drop of water can flow down towards several minima. For applications
to image segmentation, the watershed is often computed from the gradient magnitude of an image. Therefore, the resulting contours are located on high gradient contours of the image,
which often correspond to the borders of the objects of interest.

Following the intuitive drop of water principle presented in the previous paragraph, the watershed cuts, a notion of a watershed in edge-weighted graphs, were introduced in \cite{cousty-et-al_pami2009}. A watershed cut is indeed a graph cut: it is only made of edges, and it partitions the vertex set of the underlying graph. The chosen graph cut satisfies two principles: (i) the induced partition associates exactly one region to each regional minimum of the graph; and (ii) each edge of the graph cut satisfies the intuitive drop of water principle presented in the previous paragraph (\ie{} a ``drop of water" falling down on this location can flow down with a descending path towards two distinct regional minima of the graph). In application to image analysis, the considered edge-weighted graph~$(V,E,w)$ is often chosen such that:
\begin{itemize}
    \item each vertex in~$V$ is a pixel or voxel of the image which is to be segmented
    \item each edge in~$E$ is a pair of pixels which are neighbor for a usual adjacency relation such as the 4- or 8- adjacency relation or the 6-, 18-, or 26- adjacency in 3D
    \item the weight $w(\{x,y\})$ of an edge linking two vertices~$x$ and~$y$ measures the dissimilarity between the two pixels linked by the edge. A simple and often used case consists of taking any distance between the brightness, color, or features associated to the pixels~$x$ and~$y$. Recently, it has become common to use contour maps or edge detection produced by deep learning models to weight the edges, as these models are able to account for higher-order semantic information (see \cite{liu-et-al_cvpr2019,soria-et-al_PR2023,vonGioi-randall_ipol2022}). A good choice is often a (linear) combination of learned gradients and local dissimilarities in order to capture multiscale information. 
\end{itemize}
The notions introduced in this section are illustrated in \Xfig{fig:WC}.

\begin{figure}[htb]
  \begin{center}
        \includegraphics[width=0.49\linewidth]{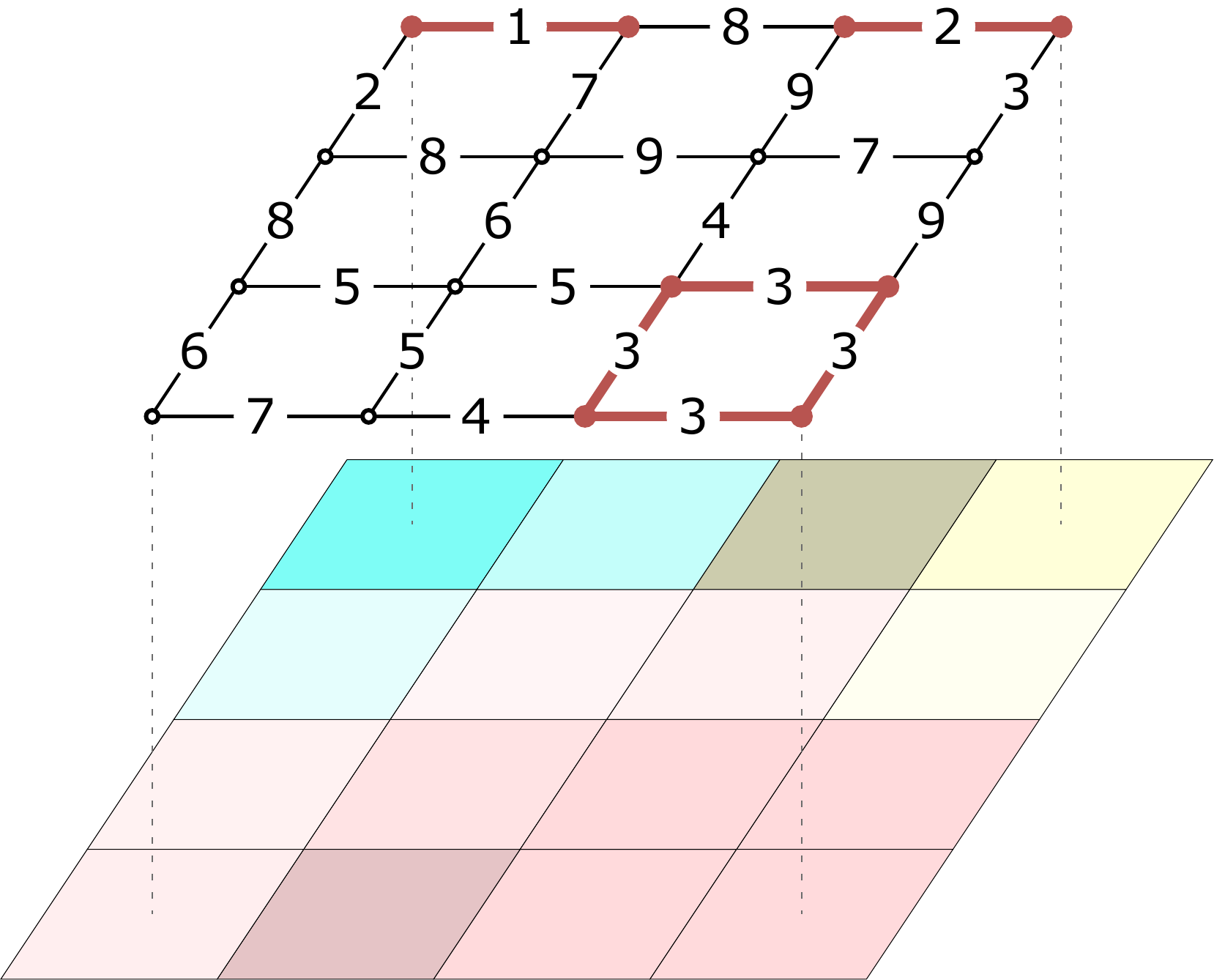}
        \includegraphics[width=0.49\linewidth]{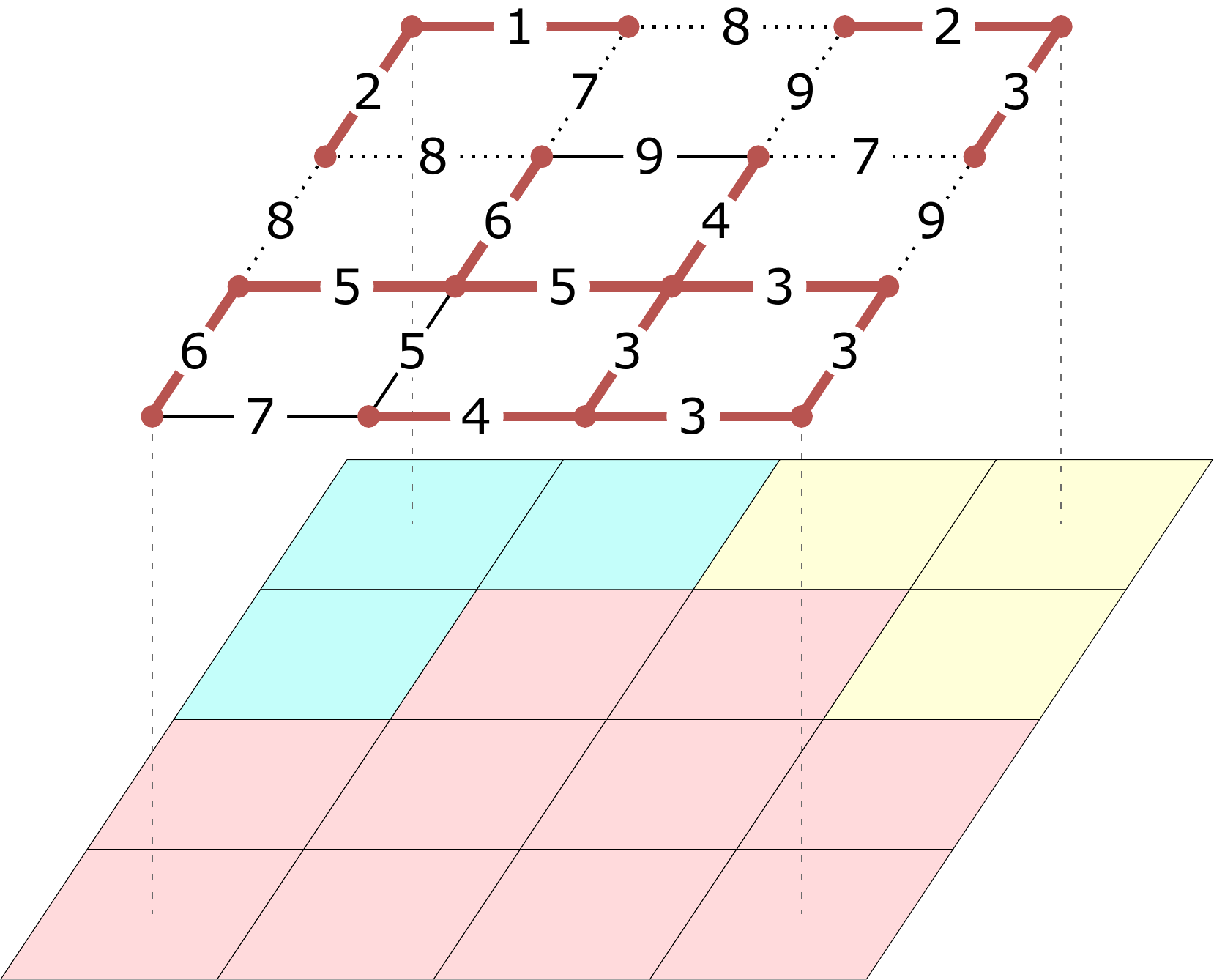}\\
    \caption{ \label{fig:WC} \textbf{Left}:  A digital image (bottom) and an edge-weighted graph representing this image (top): vertices correspond to pixels and edges represent the adjacency between pixels. The edge weights measure the dissimilarity between the values of adjacent pixels. The three minima of the edge-weighted graph are represented in bold red. \textbf{Right}: a watershed cut of the edge-weighted graph, which is represented by dashed edges in the graph (top) and by a color label in the image domain (bottom). The minimum spanning forest rooted in the minima of the graph associated to the watershed is depicted in bold red.}
  \end{center}
\end{figure}

Having such a weighted graph $G = (V,E,w)$, let us now formalize the three main notions necessary to define watershed cuts.
\begin{itemize}
    \item {\bf Regional minimum.} Intuitively, a regional minimum of~$G$ is a subgraph~$X$ of constant weight from which it is not possible to escape without climbing. Formally, the subgraph~$X$ is a {\em minimum of~$G$} if every two edges of~$X$ have the same weight and every edge adjacent to an edge of~$X$ has a weight strictly higher than the weight of an edge of~$X$.
    \item{\bf Graph cut.} Any partition of the vertex set of a graph which is made of connected regions is called a {\em graph cut} and any graph cut~$C$ induces a unique {\em cut set} composed of every edge of~$G$ linking two end-vertices in two distinct regions of the partition~$C$.
    \item {\bf Drop of water principle.} Intuitively, a graph cut satisfies the {\em drop of water principle} if from every edge in the cut set it is possible to reach two distinct minima of~$G$ by a path of descending weights, more precisely, if for any edge~$\{x,y\}$ in the cut set, there exists two paths~$\pi_x = (x_0=x, \ldots, x_k)$ and~$\pi_y = (y_0=y, \ldots,  y_\ell)$ such that:
    \begin{itemize}
        \item $\pi_x$ and~$\pi_y$ are paths of descending weights that do not cross the cut set;
        \item the weight of the first edge of $\pi_x$ (resp. $\pi_y$) is not greater than the weight of~$\{x,y\}$;
        \item $x_k$ and~$y_\ell$ belong to two distinct minima of~$G$.
    \end{itemize}
\end{itemize}

\begin{defn}[watershed cut]
    A {\em watershed cut of~$G$} is a graph cut $C$ which is such that each region of~$C$ contains exactly (the vertex set of) one regional minimum of~$G$, and which satisfies the drop of water principle.
\end{defn}

Watershed cuts have been deeply studied in the literature and are known to satisfy several properties establishing their consistency or combinatorial optimality \cite{cousty-et-al_pami2008,allene-et-al_ivc2010,couprie-et-al_pami2010} and drawing links with topology-preserving transforms \cite{cousty-et-al_pami2009,cousty2014collapses}, shortest paths or fuzzy connectedness \cite{cousty-et-al_pami2009}. In our context, it is important to recall a characterization of watershed cuts based on the minimum spanning forest/tree problems, as the algorithms that we are going to propose in further sections are based on the famous Kruskal's minimum spanning tree algorithm.

\begin{thm}[optimality \cite{cousty-et-al_pami2008}]
\label{thm:WCOptimality}
 A cut~$C$ is a watershed cut of~$G$ if and only if there exists a minimum spanning forest~$F$ of~$G$ which is rooted in the regional minima of~$G$ and whose connected component partition is equal to~$C$, more precisely, if and only if there exists a subgraph~$F$ of~$G$ such that:
 \begin{enumerate}
    \item each connected component of $F$ includes exactly (the vertex set of) one regional minimum of~$G$;
    \item the connected component partition of~$F$ is $C$; and
    \item the sum of the weights of the edges in~$F$ is minimum among the sums of the weights of the subgraphs of~$G$ that satisfy the above two conditions.
 \end{enumerate}
\end{thm}

\subsection{Seeded watershed cuts}
\label{sec:seededWatershedCut}
In practice, it is often desirable to supervise the segmentation process by providing a set of markers or seeds such that the resulting partition contains exactly one region associated with each marker that is provided. In this case, each marker is considered a clue to the location of the object of interest to be segmented.

Intuitively, as early described by Beucher and Meyer in \cite{beucher-meyer_1992}, such seeded segmentation can be obtained by considering the watershed cut of a modified version of the graph~$G$. The idea is then to consider a modified weight map~$w'$ such that the minima of~$w'$ correspond exactly to the provided markers and such that $w'$ is ``as close as possible" to~$w$. This filtering $w'$ can be obtained in two steps. First, the process begins by ``digging" (\ie{} setting to a minimal value) the weights of the edges that belong to one of the markers. Then, a closing by reconstruction is performed. The result of this closing by reconstruction, a filtering process studied in mathematical morphology, can be obtained using a data structure called a component tree. It consists of raising the weights of the undesired minima until they vanish. The watershed cut of the modified version $w'$ of~$w$ is called a watershed cut of~$w$ seeded in the provided marker set~$\mathcal{M}$. It indeed contains one and only one region for each of the markers. The notion of a seeded watershed cut (and also Theorem~\ref{thm:seededWC}) is illustrated in \Xfig{fig:seededWC}.

\begin{figure}[htb]
  \begin{center}
        \includegraphics[width=1\linewidth]{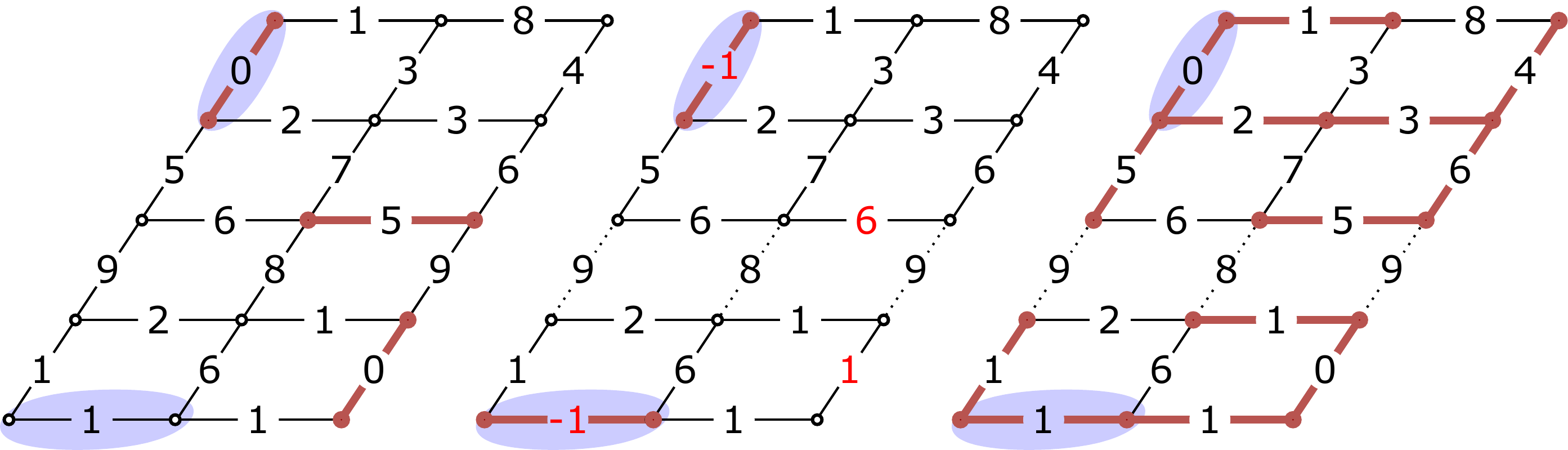}
    \caption{\label{fig:seededWC} \textbf{Left}:  an initial edge-weighted graph with three minima in bold red and two seeds overlaid in blue. \textbf{Middle}: a closing by reconstruction from the initial graph and the seeds (changed weights in red). Edge weights within seeds are set to -1, forcing them to become minima. The closing by reconstruction then fills the minima of the original graph. Finally, the associated watershed cut is depicted by dashed edges. \textbf{Right}: a minimum spanning forest (bold red) of the initial graph rooted in the seeds, together with its induced graph cut (dashed edges), which coincides with the watershed of the closing by reconstruction in the middle figure, as stated in Theorem~\ref{thm:seededWC}.
}
  \end{center}
\end{figure}

A formal definition of seeded watershed cut can be found in \cite{allene-et-al_ivc2010}. It is possible to obtain the result of a seeded watershed cut without actually computing the modified version~$w'$ of~$w$. Instead of composing filtering followed by watershed cuts, it is sufficient to consider the cut induced by a minimum spanning forest rooted in the provided markers. More precisely, the following property has been proved in \cite{allene-et-al_ivc2010}.

\begin{thm}[seeded MSF cuts, \cite{allene-et-al_ivc2010}]
\label{thm:seededWC}
 Let~$\mathcal{M}$ be a marker set. If~$C$ is the connected component partition of a minimum spanning forest of~$G$ rooted in the marker set~$\mathcal{M}$, then $C$ is a watershed cut of~$G$ for the marker set~$\mathcal{M}$.
\end{thm}

For the sake of completeness, we recall that a {\em minimum spanning forest of~$G$ rooted in~$\mathcal{M}$} is a subgraph~$F$ of~$G$ such that:
\begin{enumerate}
    \item each connected component of~$F$ includes exactly one of the elements of~$\mathcal{M}$;
    \item $F$ spans~$G$, \ie{} every vertex of~$G$ is a vertex of~$F$; and
    \item the sum of the weights of the edges in $F$ is minimal, that is to say, the sum of the edge weights in any subgraph of~$G$ for which the two above conditions hold true is greater than or equal to the sum of the weights of the edges in $F$.
\end{enumerate}

\subsection{Hierarchical watershed cuts}

\label{sec:subsection_hierarchical_watershed_cuts}

In some cases, the objects of interest do not all appear at the same scale, and some of them can be nested into each other. This is, for instance, the case of the nose, the face, and the silhouette of a person in a picture. In such cases, a multiscale or hierarchical image segmentation is desirable. Early studies of this approach date back to \cite{beucher_ismm1994} and \cite{najman-schmitt_pami1996} in mathematical morphology and to \cite{guigues-et-al_ijcv2006}, under the name of scale-set theory, from an optimization perspective. Such a hierarchical representation of an image can be envisioned thanks to hierarchical watershed cuts \cite{cousty-najman_ismm2011,cousty-et-al_ismm2013,maia-et-al_prl2019,maia-et-al_jmiv2020} as presented in this section. 

From a high-level topographical point of view, a {\em hierarchical watershed} can be obtained by:
\begin{enumerate}
\item considering a {\bf series~$\Phi = (\Phi_0(G)=G, \ldots, \Phi_N(G))$ of progressive filtrations} (described below) of the input weighted graph $G$ seen as a relief, for any $N \geq 1$, ordered from the lightest to the most aggressive; and
\item considering an associated {\bf series~$(C_0, \ldots, C_N)$ of watershed cuts} of these filtrations where~$C_\lambda$ is a watershed of~$\Phi_\lambda(G)$. 
\end{enumerate}
The obtained series~$(C_0, \ldots, C_N)$ of partitions is called a {\em hierarchical watershed} if the following conditions are met:
\begin{itemize}
\item{\bf Hierarchy principle (\cite{benzecri-et-al_dunod1973,leclerc_1981,guigues-et-al_ijcv2006,ronse_jmiv2014}).} At every scale, every region can be obtained as the merging of regions at the previous scale. More formally, for every~$\lambda$ in~$\{1,\ldots, N\}$, the cuts~$C_{\lambda-1}$ and $C_{\lambda}$ are {\em nested}, meaning that, each region~$R$ of~$C_{\lambda-1}$ is included in a region $R'$ of~$C_{\lambda}$, hence, conversely, every region~$R'$ in $C_{\lambda}$ is the union of the regions of~$C_{\lambda-1}$ which are included in~$R'$ ($R' = \cup \{R \suchthat R \in C_{\lambda-1}, R \subseteq R'\}$).
\item {\bf Extensive connected operators.} Intuitively, the filtration process involved in hierarchical watershed acts by progressively removing (or ``filling in") the minima while leaving unchanged the other weights of the graphs corresponding to the highest mountains and crest lines. The preservation/removal of a minimum in a filtration is in general driven by a measure of importance of that minimum like its area, depth or more frequently the area, depth or volume of its drainage region. Such filters are known as connected extensive operators in mathematical morphology \cite{salembier-serra_tip1995, vachier-meyer_nsip1995, salembier-wilkinson_sigprocmag2009,breen-jones_cviu1996}. More formally, given an increasing measure~$\mu$ ({\em e.g.}, the size) mapping a positive real value to any subgraph and a scale parameter~$\lambda$, the filtering result is obtained by: 
\begin{enumerate} 
\item considering the set of connected components~$C(G,w,k)$ of the lower-threshold of the weights $w$ at every possible value~$k$ (\ie{} the graph obtained by removing from~$G$ the edges of weight greater than~$k$); 
\item removing from the set $C(G,w,k)$ every connected component for which the measure~$\mu$ is less than the scale parameter~$\lambda$, yielding the set $C'(G,w,k)$;
\item building the filtered map~$w'$ by stacking the connected components remaining in the sets $C'(G,w,k)$, \ie{} setting, for every edge~$u$ of~$G$, $w'(u) = \min\{ k \suchthat u$~is an edge of a connected component in~ $C'(G,w,k)\}$. In the following, any considered filtration series results from an extensive connected operator.
\end{enumerate}
\end{itemize}

The notion of a hierarchical watershed is illustrated in \Xfig{fig:HierarchicalWC}.

\begin{figure}[htb]
  \begin{center}
  \begin{tabular*}{1\linewidth}{@{\extracolsep{\fill}}c c}
        \subfigure[$\Phi_0(G) =G$]{
        \includegraphics[width=0.47
        \linewidth]{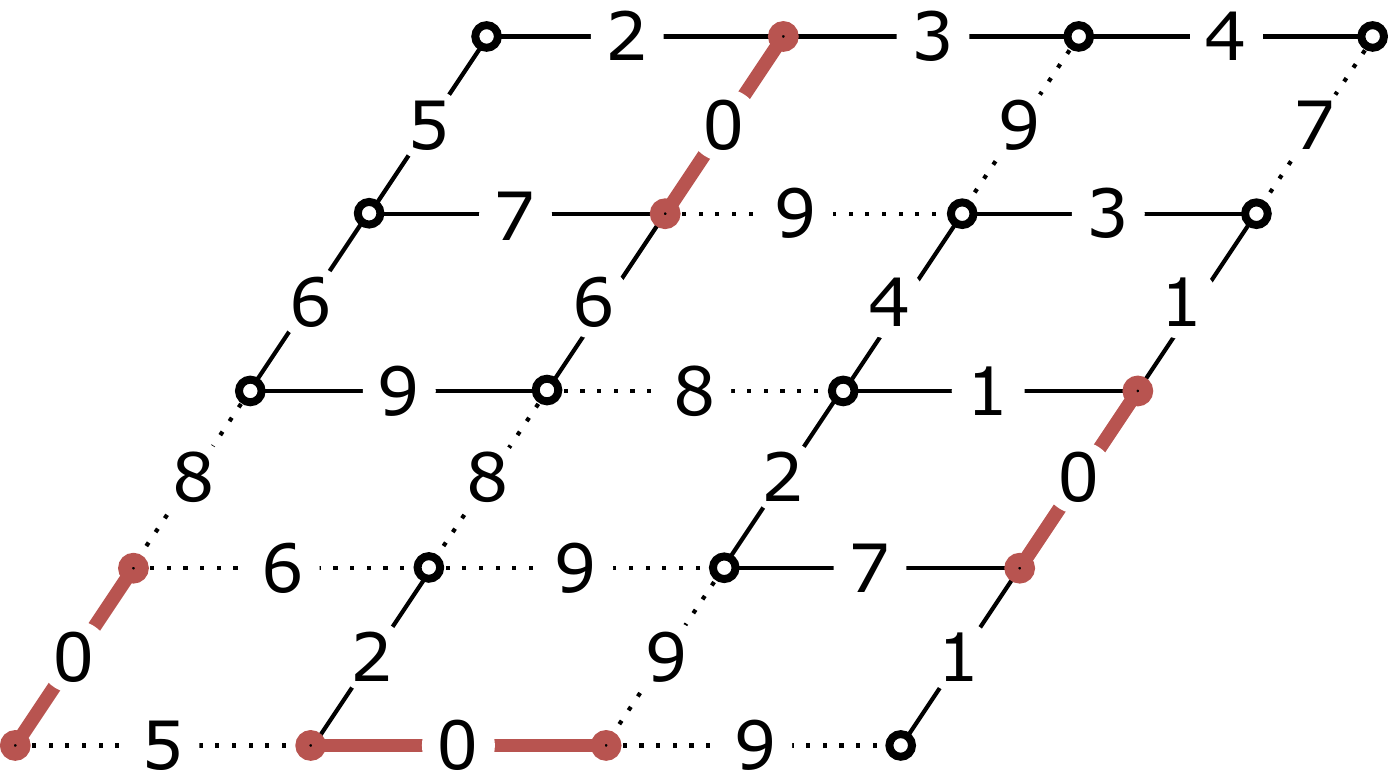}}&
        \subfigure[$\Phi_1(G)$]{        \includegraphics[width=0.47\linewidth]{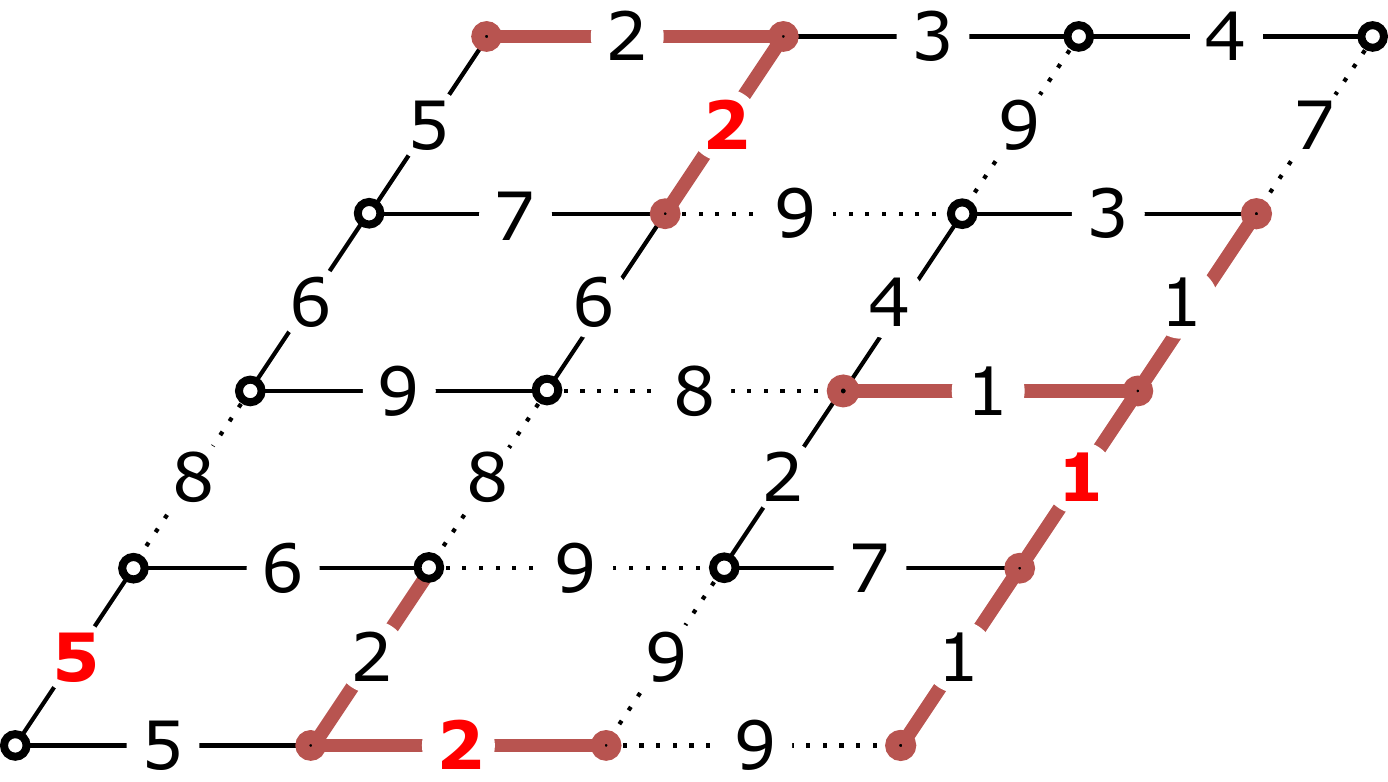}}\\
        \subfigure[$\Phi_2(G)$]{        \includegraphics[width=0.47\linewidth]{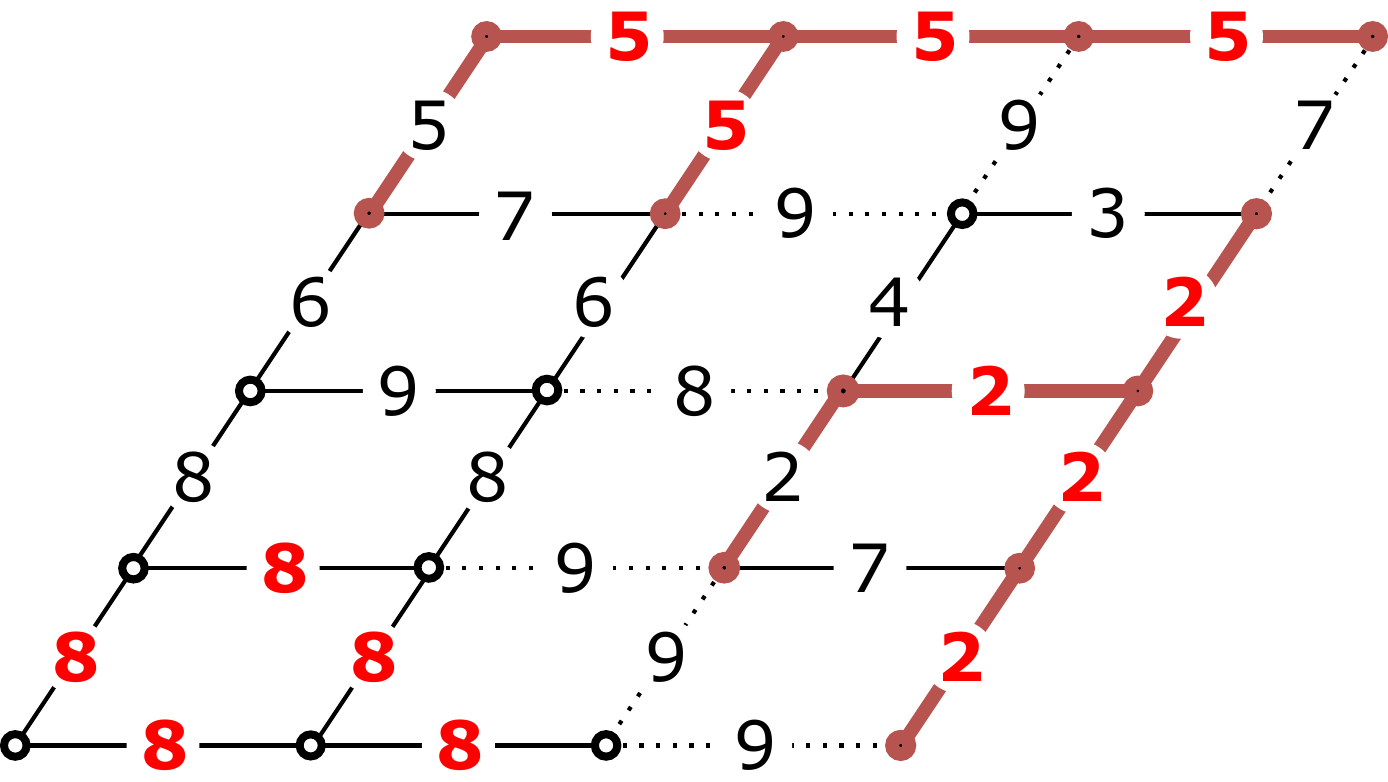}}&
        \subfigure[$\Phi_3(G)$]{        \includegraphics[width=0.47\linewidth]{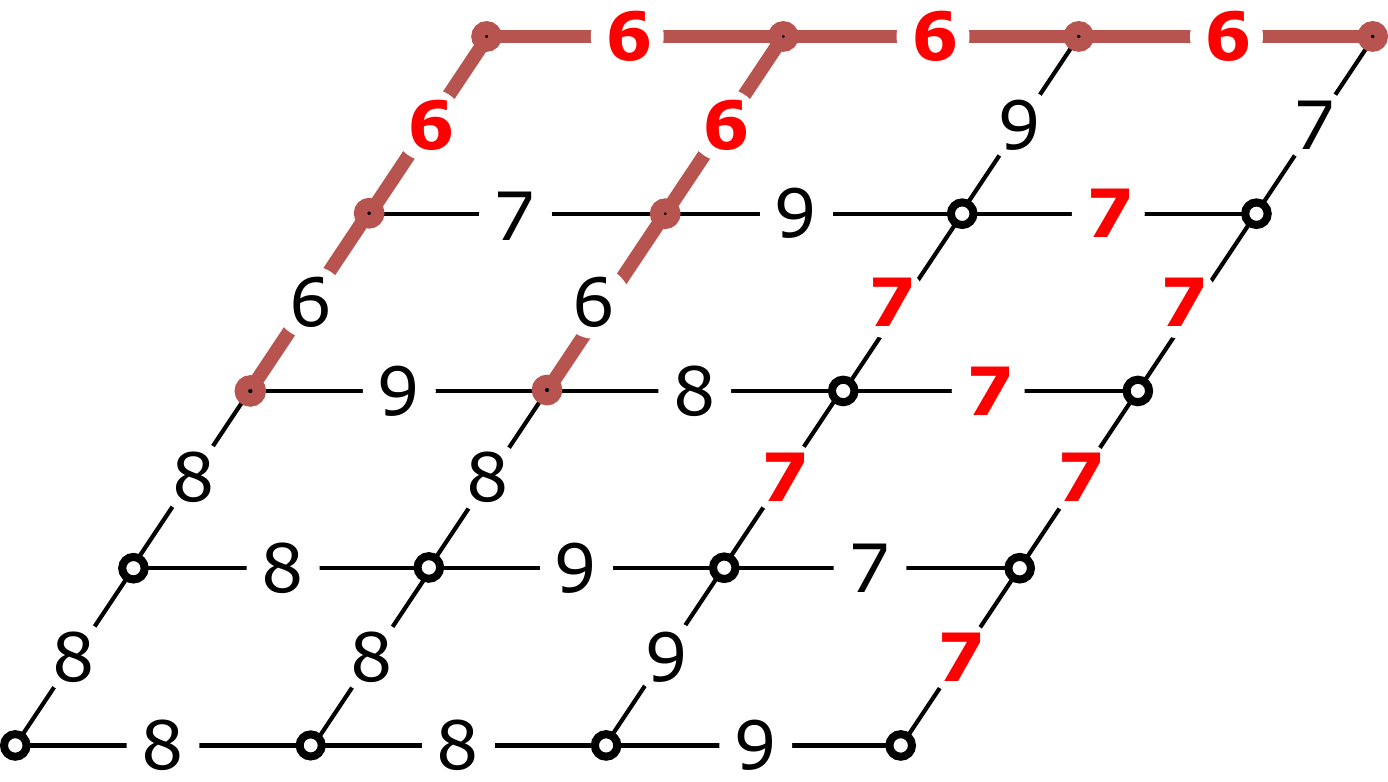}}
   \end{tabular*}
   \subfigure[Hierarchical watersheds]{        \includegraphics[width=0.5\linewidth]{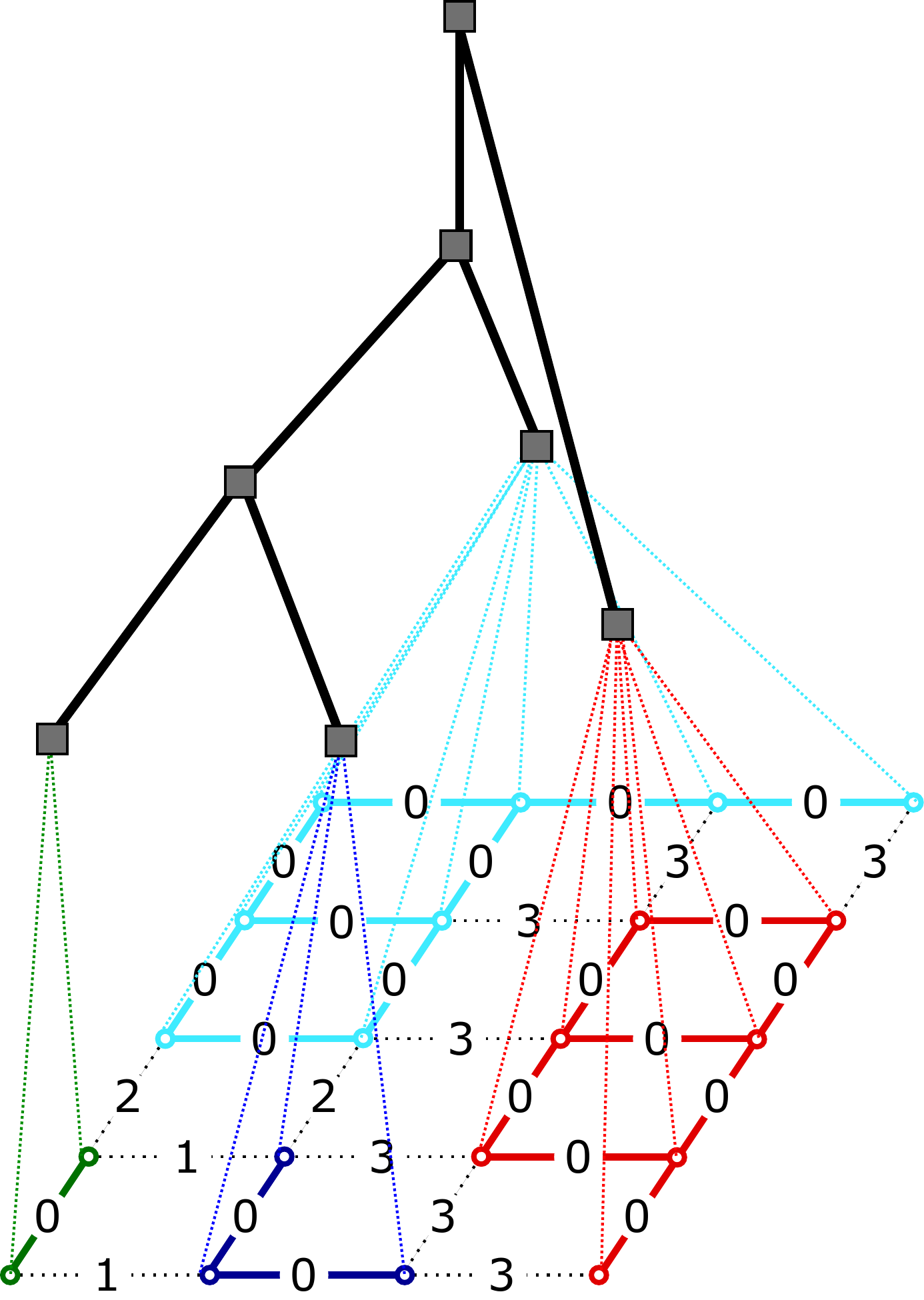}}
    \caption{ \label{fig:HierarchicalWC}  (a-d): Four filtrations of a graph obtained with a connected closing by area removing successively the components with less than 0, 3, 6, and 8 vertices. The edge weights that changed between two levels of filtrations are highlighted in red. The minima are in bold red. The watershed cuts of each filtration are represented by dashed edges. Observe that the watershed cut obtained at one filtration level is included in the watershed cut obtained at a lower level. (e): The watershed cuts of the series of filtrations form a hierarchical watershed. It can be represented as a saliency map (bottom), \ie{} an edge-weighted graph, where the weight of an edge indicates at which filtration level it ceases to be in the cut or 0 if it does not belong to any watershed cut. Note that every watershed cut from the series of filtration can be recovered by thresholding the saliency map. The four regional minima of the original graph are the four minima, all at level 0, of the saliency map (each depicted in a different color). The hierarchical watershed can be equivalently represented  as a tree (top) representing the fusion of regions along the filtration process.
}
  \end{center}
\end{figure}

As in the single-scale case, the hierarchical watershed segmentation problem is deeply related to the minimum spanning tree combinatorial optimization problem. This link relates to the extinction value of the minima. More precisely, for each minimum of the graph, we can associate an {\bf extinction value}~$\epsilon(M)$ which is the highest scale parameter~$k$ such that the minimum~$M$ ``remains" (\ie~$M$ is included in a minimum) in the filtration~$\Phi_k(G)$. Let us state the main result of the hierarchical watershed theory which establishes the multiscale optimality of the approach and suggests the algorithmic framework developed in this article.

\begin{thm}[multi-scale optimality]
\label{thm:multi-scaleOptimality}
Let~$\Phi = (\Phi_0(G)=G, \ldots, \Phi_N(G))$ be a filtration series of~$G$ and let~$\mathcal{H} = (C_0, \ldots, C_N)$ be a series of partitions which satisfies the hierarchy principle. If, 
\begin{itemize} 
\item[(i)]  for any value~$\lambda \in \{0, \ldots, N\}$, there exists a minimum spanning forest~$F_\lambda$ of $G$ which is rooted in the minima of~$G$ of extinction value (with respect to the filtration series~$\Phi$) greater than or equal to~$\lambda$ and whose connected component partition is~$C_\lambda$,
\end{itemize}
then
\begin{itemize} 
\item[(ii)] $\mathcal{H}$ is a hierarchical watershed with respect to the filtration~$\Phi$.
\end{itemize}
Furthermore, if~$\Phi$ is a one-by-one minima filtration series (\ie{} the number of minima decreases by one between any two successive filtrations), the converse is also true.
\end{thm}

The proof of Th.~\ref{thm:multi-scaleOptimality} is provided in Appendix~\ref{sec:ProofTheo4}.

Besides establishing the optimality of hierarchical watershed cuts, the previous theorem provides us with the following important observations:
\begin{itemize}
\item it is not necessary to actually compute the successive filtrations of~$G$ and to extract the watershed of these filtrations, wishing them to satisfy the hierarchy principle,  
\item instead, one only needs to know the extinction values of every minimum and to compute nested minimum spanning forests of the original graph~$G$.
\end{itemize}

In the following, we will use the latter approach, showing how to
\begin{itemize}
    \item compute extinction values of  minima (or more generally seeds); and
    \item compute successive minimum spanning forests based on Kruskal's minimum spanning tree algorithm and on the associated BPTAO.
\end{itemize}

\subsection{Watershed cuts from a hierarchy}
\label{sec:seeded_hierarchical_watershed}

In the previous subsection, we have seen that the watershed cuts may be used to define multiscale image representation. On the other hand, in Section~\ref{sec:seededWatershedCut}, we have shown how watershed can be used to extract a segmentation directly from an image based on user (or machine learning) supervision given in the form of seed vertices indicating clues on objects of interest location. In this subsection, we explain how watershed cuts can be used to extract a segmentation from a hierarchy and a seed-based supervision. The core idea is to reduce the latter extraction to the former one. This can be done with saliency maps (\cite{najman-schmitt_pami1996,cousty-et-al_jmiv2018}, a notion also known as ultrametric contour map \cite{arbelaez_cvprw2006}), which is an equivalent representation of a hierarchy in the form of an edge-weighted graph that can be seen as the characteristic function of the hierarchy.

Intuitively, the {\bf saliency map} of a hierarchical watershed cut is a new weight function for the edges of the graph~$G$ such that the weight of each edge indicates the highest scale for which the two extremities of that edge belong to distinct regions. In other words, at any scale below the saliency weight of an edge, its two extremities lie in distinct regions of the hierarchy, whereas at larger scales they belong to the same region. In \Xfig{fig:HierarchicalWC}(e), one can see the saliency map of the presented hierarchical watershed. It has to be noted that the saliency map is an equivalent representation of a hierarchical watershed. Using notions from the classification community, the saliency weight of an edge is given by the distance between its extremities, where the distance is the ultrametric  \cite{leclerc_1981} associated to the hierarchical watershed.

In the remaining part of this article, {\bf seeded hierarchical watershed cuts} refer to the results of the following supervised segmentation scheme:
\begin{enumerate}
\item compute a hierarchical watershed cut $(C_0, \ldots, C_N)$ of~$G$;
\item compute the saliency map~$w_s$ of~$(C_0, \ldots, C_N)$;
\item compute the watershed cut of the saliency-weighed graph~$(V,E,w_s)$ for the given marker set. 
\end{enumerate}

\section{Algorithmic scheme overview}
\label{sec:algoOverview}

In this section, we present an overview of {\color{black}the proposed framework to produce watershed-cut segmentations from a given image, which comprises three main steps: (i) construction of the Binary Partition Tree by Altitude Ordering (BPTAO), that allows one to efficiently browse the minimum spanning tree edges of the graph; (ii) detection of the edges of the MST where to cut in order to obtain the segmentation, with the possible scale at which the corresponding cut occurs; these edges are detected by searching particular nodes of the BPTAO; (iii) computation of the segmentation regions given by the cuts.}

{\color{black} Our framework is based on the link between BPTAO and MSTs, which have already been studied in \cite{cousty-et-al_ismm2013}, with efficient algorithms to compute BPTAO proposed in \cite{najman-et-al_ismm2013} and implemented in \cite{PCCGKN:softwarex2019}. In simple terms, the BPTAO is a tree-based hierarchical data structure used to represent an MST and can be seen as a by-product of the efficient Kruskal’s MST algorithm. Figure \ref{fig:bpt_mst} shows the BPTAO of an edge-weighted graph. In particular, the leaves of the BPTAO are associated to the vertices of the MST and each non-leaf node is associated to an edge of the MST. Intuitively, we can say that every non-leaf node of the BPTAO represents the addition of an edge in a non-decreasing order of their weights (indicated by the red boxes in Figure \ref{fig:bpt_mst}) to the MST during Kruskal’s algorithm, where the added edge is used to merge the two connected regions that contain the edge extremities.}

\begin{figure}
    \centering
    \includegraphics[width=1.0\linewidth]{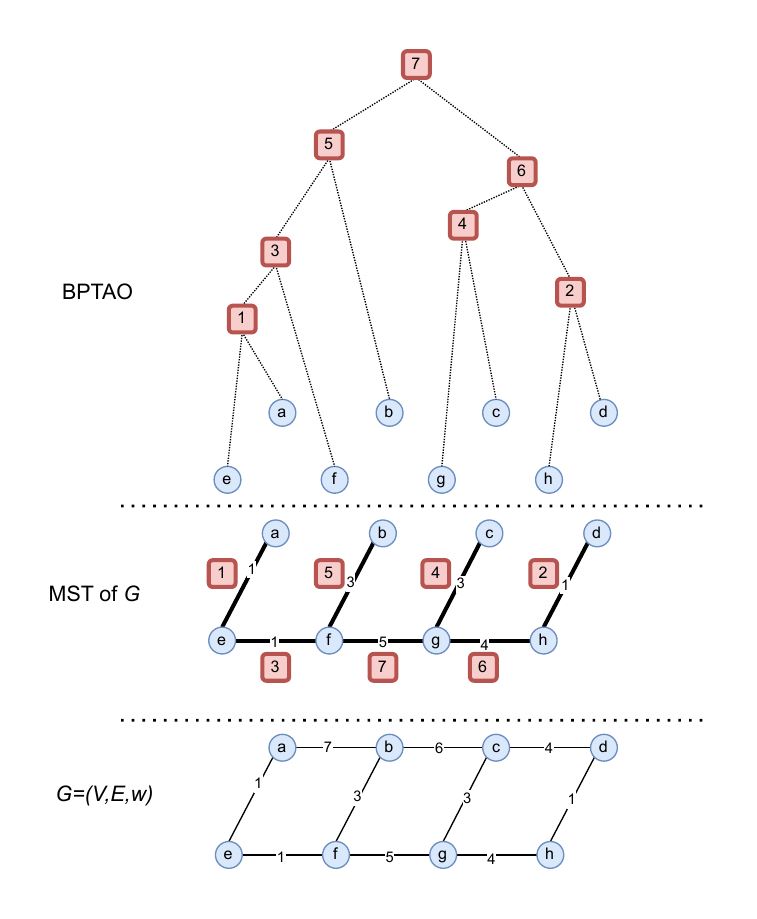}
    \caption{\color{black} A weighted graph $G=(V,E,w)$, a MST of $G$ with a non-decreasing ordering on its edge weights indicated in red, and the BPTAO of this MST for this given ordering.}
    \label{fig:bpt_mst}
\end{figure}

The overall steps of our algorithm are illustrated in Figure \ref{fig:schema_section3} and described below.

\begin{figure*}[!htbp]
    \centering
    \includegraphics[width=0.6\textwidth]{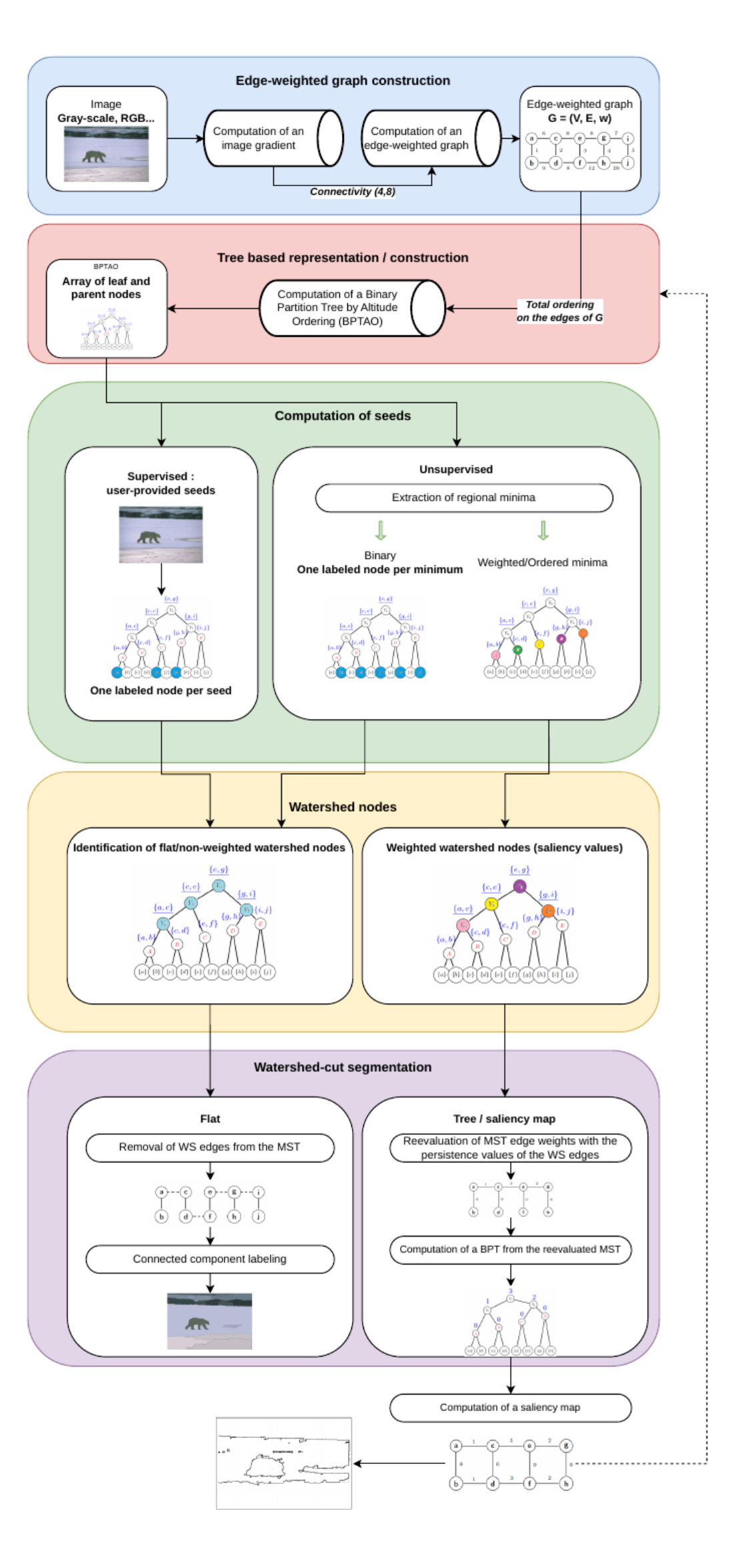}
    \caption{Algorithmic overview of (hierarchical) watershed-cut segmentation.}%
    \label{fig:schema_section3}
\end{figure*}

\begin{enumerate}

    \item {\bf Computation of an edge-weighted graph $G=(V,E,w)$ for an input image $I$.} As shown in Section \ref{sec:bpt}, BPTAOs are computed from (connected edge-)weighted graphs, {\color{black}commonly obtained by setting $V$ as the set of pixels of $I$ and $E$ as edges between the 4- or 8-connected pixels.}
    \item {\bf Construction of a BPTAO from $G=(V,E,w)$ (Section \ref{sec:bpt}).} In simple terms, the nodes of the BPTAO of a given weighted graph $G$ correspond to the connected components produced during Kruskal’s minimum-spanning-tree algorithm, the leaf nodes being the vertices of $G$. Given a set $V$ of vertices (leaf nodes) and a total ordering $\prec$ on $E$, we compute a sequence of increasingly coarser partitions of $V$ by incrementally merging the regions linked by the edges $e_1, e_2, ..., e_{\|E\|}$ (ordered according to $\prec$) until a single region has been obtained.
    \item {\bf Selection of seeds/markers (Section \ref{sec:selection_seeds}).} The construction of (hierarchical) watershed segmentations is guided by seeds indicating {\color{black}where and} how many regions (or hierarchical levels) should be present in the final segmentation. Those seeds can be obtained through unsupervised or supervised methods, with the set of (ordered) regional minima of $G$ and seeds manually input by the user (drawn from the original image) being among the common choices. {\color{black} Namely, in the case of hierarchical watersheds, the set of regional minima is often ordered according to their ``importance level'' for a given regional attribute, such as the area (proportional to the number of pixels in a region) or volume (proportional to the number and intensity of pixels in a region), which is commonly measured by extinction values \cite{vachier-meyer_nsip1995}. In turn, extinction values can also be obtained in linear time from the BPTAO, as described later in Section \ref{subsection:seeds_hierarchical}.}
    \item {\bf Computation of watershed nodes of BPTAO or saliency values of the MST edges (Section \ref{sec:extraction_watershed_cut})}
    \begin{itemize}
        \item {\bf Non-hierarchical case:} once the BPTAO has been constructed, the {\color{black}watershed-cut edges} of a MST of $G$ can be obtained in a single pass from the leaf nodes to the root of BPTAO. Those {\color{black}watershed-cut edges} establish a partition on the vertices of $G$. %
        \item {\bf Hierarchical case:} in the hierarchical version of watersheds, where weighted or ordered markers are given as input, we also need to obtain the saliency values of watershed-cut edges, from which a nested sequence of partitions of $V$ can be derived. In this case, the saliency values play a role in specifying at which level of the hierarchical watershed the nodes connected by those edges should be merged. %
    \end{itemize} 
    \item {\bf Computation of connected components corresponding to a (hierarchical) watershed of $G$ for the given seeds $S$. (Section \ref{sec:extraction_cc})} Once the watershed nodes and, possibly, their saliency values, have been computed, the following step computes the (hierarchical) watershed segmentation of $S$, which consists of a (hierarchical set of nested) partition(s) of $V$. In the non-hierarchical case, if we are rather interested in the contours separating those connected components, a cut of $G$ can be later computed in linear time in a single pass on the set $E$. In the hierarchical case, {\color{black}the saliency map of the MST can be obtained in linear time with respect to $\cardinal{E}$ using Algorithm \ref{alg:hwatersheds_seeds} introduced later in Section~\ref{sec:hierarchical_watershed}}. %
\end{enumerate}

{\color{black}In the remaining part of this article, we assume that seeds are composed of a single pixel. Moreover, in our algorithms, we will not consider cases where a single region in the watershed segmentation contains multiple non-connected seeds. Though those cases may happen in practice, they can be handled as a post-processing step by merging neighboring regions labeled by the same seeds.}

\section{Building a Binary Partition Tree}
\label{sec:bpt}

The Binary Partition Tree by Altitude Ordering (BPTAO) of an edge-weighted graph $G=(V,E,w)$ can be efficiently built by a modified version of Kruskal's minimum spanning tree (MST) algorithm. 
Indeed, Kruskal's MST algorithm proceeds as follows:
\begin{enumerate}
    \item At the start, the vertices of the graph are partitioned into singletons and the (incomplete)  MST is empty;
    \item Then, the edges of the graph are browsed in increasing weight order. Each time we find an edge $\mset{x,y}$ linking two vertices $x$ and $y$ belonging to two different regions, this edge is added to the MST and the two regions containing $x$ and $y$ are merged.
    \item When all the vertices have been merged into a single region, the MST is complete, and the algorithm stops.
\end{enumerate}
Thus, each step of the algorithm is associated with a partition of the vertices of the graph and the partition of step $n+1$ is obtained by merging two regions of the partition of the step $n$ (through an edge of the MST): in other words, the partition in step $n$ is a refinement of the partition at step $n+1$. The BPTAO is indeed exactly equal to this sequence of partitions\Cite{cousty-et-al_ismm2013} and one way to obtain it is to \emph{record} the region merges during the execution of Kruskal's algorithm. This process is illustrated in Figure~\ref{fig:BPTAO}.

\newcommand{\taille}{0.45}
\begin{figure}[tb]
  \begin{center}
    \begin{tabular*}{1\linewidth}{@{\extracolsep{\fill}}c c}
      \subfigure[Edge-weighted graph]{
        \includegraphics[width=\taille\linewidth]{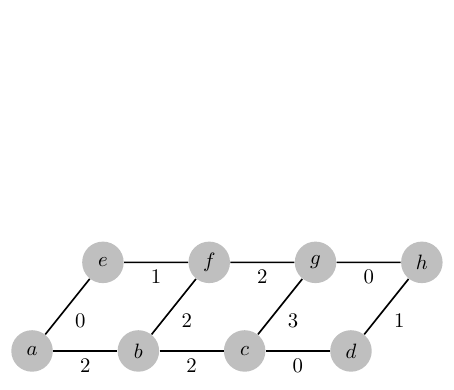}}&
      \subfigure[First nodes of the BPTAO]{
        \includegraphics[width=\taille\linewidth]{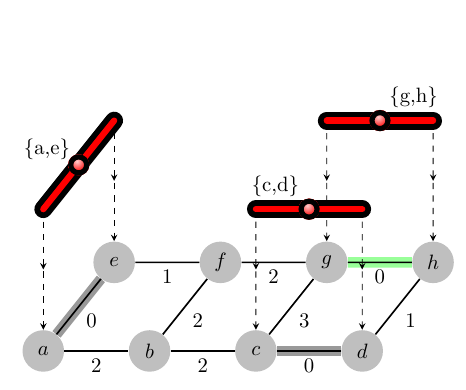}}\\
      \subfigure[Adding a node to the BPTAO]{
        \includegraphics[width=\taille\linewidth]{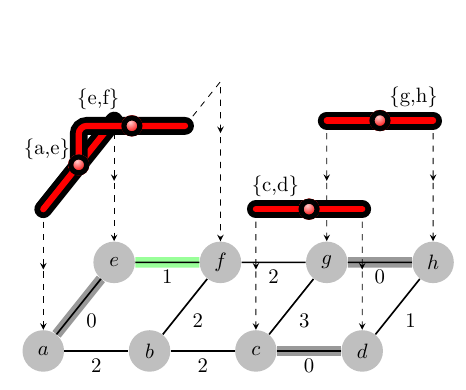}}&
      \subfigure[Final BPTAO]{
        \includegraphics[width=\taille\linewidth]{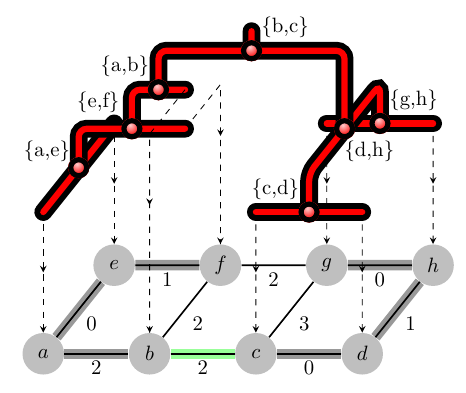}}\\
    \end{tabular*}
    \caption{\label{fig:BPTAO} Some steps of the construction of the MST of an edge-weighted graph with Kruskal's algorithm and the corresponding BPTAO. Top left: initial edge-weighted graph: each vertex of the graph represents a region. Top right: the algorithm first processes the edges $\mset{a,e}$, $\mset{g,h}$ and $\mset{c,d}$ of weight 0 in any arbitrary order. The vertices corresponding to those edges are merged: this is represented by the red trees whose root nodes indicate the edge of the MST where the merge occurred. Bottom left: then the edge $\mset{e,f}$ is processed. The region $\mset{a,e}$ is merged with the singleton $\mset{f}$. Bottom right: the process continues until the edge $\mset{b,c}$ is processed, completing the MST and leading to a partition with a single region containing all the vertices.}
  \end{center}
\end{figure}

Kruskal's algorithm relies on a union-find data structure to represent the current partition of the graph vertices, and, in order to be efficient, one must use union-find with path compression and rank optimization which collapses region merging sequences: one must thus record the merge sequence in a separate data-structure.

This is done in Algorithm \ref{alg:bpt}, which implements the construction of the BPTAO based on Kruskal's algorithm. In this algorithm, we assume that the graph vertices $V$ are represented by integers from 0 to $\cardinal{V}-1$. The nodes of the constructed BPTAO will also be represented by integers from 0 to $2\cardinal{V}-1$: the $\cardinal{V}$ first nodes of this tree are indeed exactly the $\cardinal{V}$ vertices of the graph. 
The algorithm relies on a standard Tarjan's union-find data structure with path compression and union-by-rank; a possible implementation is given in Algorithm \ref{alg:union_find} for reference. 
We can recognize the global structure of Kruskal's algorithm with two new data-structures:
\begin{itemize}
    \item An array \texttt{parents} of size $2\cardinal{V}-1$: This array records the region merging operations, represented as a parent relation. For any node $i$ of the tree, \texttt{parents[i]} is the parent node of $i$ in the tree. During the execution of the algorithm, this array will indeed represent a forest of trees.
    \item An array \texttt{roots} of size $\cardinal{V}$ that maps the canonical elements of the regions in the union-find data structure to the roots of the trees in the forest represented by the array \texttt{parents}. For any vertex $v$ of the graph, \texttt{uf.find(v)} is the canonical element representing the region containing $v$ and \texttt{roots[uf.find(v)]} is the root of the tree representing the merging sequence that leads to this region. 
\end{itemize}
Note that a more detailed description of this algorithm is given in~\cite{najman-et-al_ismm2013} with an incremental transformation of Kruskal's algorithm.

\SetKwProg{Fn}{Function}{}{end}

\begin{algorithm2e}[tb]
\small
\caption{Binary partition tree by altitude ordering (BPTAO)}\label{alg:bpt} %
\KwData{An edge-weighted graph $G=(V,E,w)$}
\KwResult{Relation \texttt{parents} and MST \texttt{MSTEdges}}
numMSTEdges := 0\; %
MSTEdges := array of size $\cardinal{V}-1$\; %
UnionFind uf\; %
parents := array of size $2\cardinal{V}-1$\; %
roots := array of size $\cardinal{V}$\; %

\ForEach{vertex $v\in V$}{
uf.makeSet($v$)\; %
roots[v] := v\;%
}
\tcp{Modified Kruskal's MST algorithm}
\ForEach{ edge $\{i,j\}\in E$ by increasing weight $w(\{i,j\})$}{
    $c_i$ = uf.find($i$); $c_j$ = uf.find($j$) \;
    
    \If{$c_i \neq c_j$}{
        MSTEdges[numMSTEdges] := $\{i,j\}$\; %
        newRoot := uf.link($c_i$, $c_j$)\;%
        numNode := $\cardinal{V}$ + numMSTEdges\; %
        parents[roots[$c_i$]] := numNode\; %
        parents[roots[$c_j$]] := numNode\; %
        roots[newRoot] := numNode\;  %
        numMSTEdges += 1\; %
    }
}
parents[$2\cardinal{V}-2$] = $2\cardinal{V}-2$\tcp*{convention: parent of root is itself}
\end{algorithm2e}

\begin{algorithm2e}
\small
\caption{Union-Find}
\KwData{\texttt{parents} and \texttt{ranks} are two mappings shared by the 3 functions.}
\label{alg:union_find}

\BlankLine
\Fn{\textnormal{\texttt{makeSet}($x$)}}{
    \tcp{Create a new set containing $x$}
    parents[$x$] := $x$; ranks[$x$] := 0\;
}

\BlankLine
\Fn{\textnormal{\texttt{find}($x$)}}{
    \tcp{Finds the canonical element of the set containing $x$}
    \If{parents[$x$] $\neq x$}{
        parents[$x$] := find(parents[$x$]) \;
    }
    \Return parents[$x$]\;
}

\BlankLine
\Fn{\textnormal{\texttt{link}($c_x, c_y$)}}{ 
    \tcp{Merges the sets represented by the canonical elements $c_x$ and $c_y$}
    \If{ranks[$c_x$] $>$ ranks[$c_y$]}{
            swap($c_x$, $c_y$)\;
        }
    \If{ranks[$c_x$] $=$ ranks[$c_y$]}{
            ranks[$c_y$] += 1\;
        }
    parents[$c_x$] := $c_y$\;
    \KwRet $c_y$\;
}
\end{algorithm2e}

After the declaration of local variables (lines 1--5), the loop over each graph vertex $v$ (lines 6--9) creates a singleton $\mset{v}$ and defines that $v$ is the root of the sub-tree associated with this singleton. Then, we have the main loop of Kruskal's algorithm, where edges are considered in increasing order. For each edge $\mset{i,j}$, we first identify the two regions $c_i$ and $c_j$ containing $i$ and $j$ (line 11). 
If those regions are different (line 12), $\mset{i,j}$ is part of the MST (line 13) and we merge the 2 regions $c_i$ and $c_j$ (line 14). 
We then record the merge operation by creating a new node in the forest (line 15). This new node becomes the parent of the sub-trees associated with the regions $c_i$ and $c_j$ (lines 16--17) and the root of the sub-tree associated with the merged region is the new node (line 18). Finally, the number of edges in the current MST is incremented. 

Similar to Kruskal's algorithm, Algorithm~\ref{alg:bpt} runs in quasi-linear time $O(\cardinal{E} × \alpha(\cardinal{V}))$ (where $\alpha$ is the inverse of the single-valued Ackermann function \cite{tarjan1975efficiency}) plus the time needed to sort the edge of the graph: $O(\cardinal{E}\log(\cardinal{E}))$ in the general case or $O(\cardinal{E})$ if linear sorting is possible.

With the BPT already computed, all remaining algorithms in this paper require a fixed number of passes over this data structure, running in linear time relative to the number of nodes of the input BPTAO.

Note that at the end of Algorithm~\ref{alg:bpt}, the nodes in the \texttt{parents} array representing the tree are sorted in topological order from the leaves to the root. This means that browsing the tree nodes in (reverse) topological order is as simple as a {\bf for}-loop over the array. In the following, given a binary partition tree $T$ and a node $n$ of $T$ we will use the following notations:
\begin{enumerate}
    \item $n.par$ is the parent of $n$ (if $n \neq$ root);
    \item $n.mst\_edge$ is the MST edge associated to $n$ (for non-leaf nodes)
    \item $n.alt$ is the altitude of $n$, \emph{i.e.} the weight of $n.mst\_edge$, the MST edge associated to this node for internal nodes and 0 otherwise;
    \item $n.c1$ is the first child of $n$ (for non-leaf nodes);
    \item $n.c2$ is the second child of $n$ (for non-leaf nodes).
\end{enumerate}

\section{Selection of seeds}
\label{sec:selection_seeds}

Once the BPTAO of the input graph has been computed, the selection of seeds, which will determine the nature ({\em flat} or {\em hierarchical}) and the number of regions of the final segmentation, can be performed through various methods.

This section describes the most common methods for extracting seeds. For the sake of clarity, we consider the non-hierarchical and hierarchical cases separately.

\subsection{Seeds for non-hierarchical watershed cuts}
\label{sec:seedminima}

\begin{algorithm2e}[tb]
\small
\caption{Minima}\label{alg:minima_ann} %
\KwData{A BPTAO \texttt{T}}
\KwResult{A binary array \texttt{minima} indicating which nodes of \texttt{T} are minima.}
\ForEach{leaf node $n\in T$}{
    numMin[n] :=0\;
    minima[n] := False\;
}
\ForEach{node $n$ of $T$ from leaves (excluded) to root (excluded)}{
    numMin[n] = numMin[n.c1] + numMin[n.c2]\;
    \If{numMin[n] = 0  \texttt{and}  n.alt $\neq$ n.par.alt  }{
        minima[n] := True\;
        numMin[n] := 1\;
    }
}
\If{numMin[T.root] = 0}{
    minima[T.root] := True\;
}
\end{algorithm2e}

\begin{algorithm2e}[tb]
\small
\caption{MinimaOneLeaf}\label{alg:minima_one_leaf_ann}
\KwData{A BPTAO \texttt{T}}
\KwData{A binary array \texttt{minima} indicating which nodes of \text{T} are minima.}
\KwResult{A binary array \texttt{minVertex} which is True for one and only one leaf node in each minimum.}
\ForEach{node $n$ of $T$ from root to leaves (excluded)}{
    \If{minima[n]}{
        minima[n.c1] := True\;
    }
}
\ForEach{leaf node $n$ of $T$}{
    minVertex[n] = minima[n]\;
}
\end{algorithm2e}

\begin{itemize}
    \item {\bf Regional minima.} Placing the seeds on the graph's regional minima stems from the original definition of watershed, with the resulting graph-cut being compatible with the drop of water principle \cite{cousty-et-al_pami2008}. %
    
    An algorithm to compute the regional minima of a graph is presented in Algorithm~\ref{alg:minima_ann}. 
    This algorithm processes a BPTAO to identify nodes that represent minima, outputting a binary array indicating these nodes. It starts by initializing leaf nodes with no minima. Then, in a bottom-up traversal from leaves to the root (excluding the root), it calculates the number of minima in each node’s subtree by summing those of its children. A node is marked as a minimum if its subtree has no minima and its altitude value differs from its parent’s. Finally, the root is marked as a minimum if the entire tree has no minima. 
    Algorithm \ref{alg:minima_one_leaf_ann} (MinimaOneLeaf) processes a BPTAO and a \texttt{minima} array (precomputed by Algorithm~\ref{alg:minima_ann}) to produce a binary array \texttt{minVertex}, which marks exactly one leaf node as \texttt{True} for each minimum node in the tree. Starting from the root, it propagates the minima status of each minimum node to its first child in a top-down traversal, continuing until a leaf is reached. Then, it assigns \texttt{minVertex[n] = minima[n]} for each leaf, ensuring only the leaf reached by this propagation is marked \texttt{True} for each minimum. This approach selects a single representative leaf for each significant region (minimum) in the tree, useful for applications like image segmentation.

    \item {\bf User scribbles.} {\color{black}For certain applications, users can actively guide the segmentation algorithm by roughly drawing scribbles inside the objects of interest. However, this technique can quickly become cumbersome as the image complexity and size increase. To alleviate this problem, hybrid methods often couple user's input with automatic segmentation, such as the one presented in \cite{straehle2012seeded}, in which the users focus on the locations of ``high uncertainty'' and can interfere in case of segmentation errors.}

    \item {\bf Supervised methods.} Taking advantage of the deep learning developments {\color{black}in the last decade, Song et al.~\cite{song2018seednet} and Da Fonseca et al.~\cite{da2021new} propose supervised models for predicting user's interactive seed or scribble inputs. Given a first point in the object of interest and a point in the background, both provided by the user, Song et al.~\citep{song2018seednet} proposes a reinforcement learning model for generating a sequence of artificial user inputs in order to improve the precision of the resulting segmentation. In
    \citep{da2021new}, the authors trained a model to predict what an average user's scribble would look like and, then, input those learned scribbles into the watershed segmentation pipeline. Both methods can save user's time and could be used in conjunction with interactive segmentation methods to fix eventual segmentation errors. More recently, Yang and Gong~\cite{yang2024foundation} employed the pre-trained {\em Contrastive Language-Image Pre-training} and {\em Segment Anything Model} to generate high-quality seeds for weakly supervised semantic segmentation.
    Finally, another way of tackling seed based segmentation with the help of supervised learning methods is to, instead of learning where seeds should be placed, learn high level image contours such that the local minima are placed in the regions of interest. The idea is to leverage the pre-trained deep learning methods trained on large semantic segmentation datasets ({\em e.g.} \cite{mottaghi2014role}) to extract semantic knowledge from new input images in the form of fuzzy contours. }%
\end{itemize}

\subsection{Weighted seeds for hierarchical watershed cuts}

\label{subsection:seeds_hierarchical}

So far, we have considered non-weighted seeds with the aim of computing a flat watershed-cut segmentation, with the only parameters being the number of seeds and their location. By introducing weighted (or totally ordered) seeds, we can obtain a sequence of nested segmentations seeded in the $i$ most important seeds, for $i$ ranging from one (the coarsest segmentation level composed of a single region) to $n$ (the finest segmentation level where each seed belongs to a distinct region). The seed weights can be interpreted as their ``importance level'', where seeds with larger weights belong to the ``most important image regions'' which should be preserved at the coarsest segmentation levels. 
The main motivation for considering weighted seeds is that regions of interest are rarely present at a single segmentation level, and post-processing is often required to filter out spurious noise.

{\bf Extinction values}, as already mentioned in Section \ref{sec:subsection_hierarchical_watershed_cuts}, are a common measure for ordering the regional minima of a graph. Informally, the idea is that each regional minimum belongs to a catchment basin from which some simple properties can be extracted, including their area (number of pixels/vertices), volume (sum of pixels' intensities) and dynamics (measure of contrast between the minima).

In order to obtain weighted seeds, we first compute increasing regional attributes ({\em e.g.} area, volume or depth) on the BPTAO nodes, which are then propagated from the root to the leaf nodes, resulting in the extinction value of each minimum (as well as their ancestors in the BPTAO).

Algorithm \ref{alg:ev} describes the computation of extinction values given a BPTAO and an increasing attribute as input. Algorithm \ref{alg:ev} is linear w.r.t. $\cardinal{V}$ and can be summarized as:

    \begin{itemize}
        \item First we need to correct the attribute values because some nodes of the BPTAO, those which correspond to sub-parts of a plateau, do not really represent any real component: they are artifacts of the binary tree structures. This situation is detected by checking if the altitude of a node $n$ is equal to the altitude of its parent. In this case, its attribute value is replaced by the maximum of the attribute values of its children, \emph{i.e.} in a filtering process, this component would disappear as soon as both its children are filtered out.
        \item Then comes the extinction value computation. We start by assigning $\infty$ (or any large enough value) as the extinction value of the root node. Next, for each node $n$ from the root to the leaf-nodes (excluded), let $maxA$ be the maximum of the attribute values of the children of $n$. Then, for each child $c$ of $n$,  the extinction value of $c$ is equal to the extinction value of $n$ if its attribute value is equal to $maxA$, and the attribute value of $c$ otherwise.
    \end{itemize}

Finally, to obtain weighted seeds from extinction values, one can simply first extract binary seeds representing the minima with Algorithms~\ref{alg:minima_ann} and~\ref{alg:minima_one_leaf_ann} and then replace the seeded minima vertices by their extinction values computed with Algorithm~\ref{alg:ev} (and non-seeded vertices are assigned a null weight).

\begin{algorithm2e}[tb]
\small
\caption{extinctionValue}\label{alg:ev}
\KwData{A BPTAO \texttt{T}}
\KwData{An array \texttt{attr} representing an increasing attribute of \texttt{T}.}
\KwResult{An array \texttt{ext} representing the extinction values of the nodes of \texttt{T} for \texttt{attr}.}
\tcp{Correct attribute values for non-canonical nodes of $T$}
\ForEach{node $n$ of $T$ from leaves (excluded) to root (excluded)}{
    \If{n.alt = n.par.alt}{
        attr[n] := max(attr[n.c1], attr[n.c2])\;
    }
}
\tcp{Compute extinction values}
ext[T.root] = $\infty$\;
\ForEach{node $n$ of $T$ from root to leaves (excluded)}{
    maxA := max(attr[n.c1], attr[n.c2])\;
    \eIf{attr[n.c1] = maxA}{ext[n.c1] = ext[n]\;}{ext[n.c1] = attr[n.c1]\;}
    \eIf{attr[n.c2] = maxA}{ext[n.c2] = ext[n]\;}{ext[n.c2] = attr[n.c2]\;}
}
\end{algorithm2e}

\section{Extracting watershed-cut edges from the BPTAO}
\label{sec:extraction_watershed_cut}
In this section, we show how to compute a (weighted) watershed cut from a BPTAO (see Section~\ref{sec:bpt}) and a set of (weighted) seeds (see Section~\ref{sec:selection_seeds}).

\subsection{Watershed-cut from seeds}
\label{sec:wat_cut_seeds}

We now present Algorithm \ref{alg:wsfromseed}, that computes a watershed from an arbitrary set of seeds. The main idea is to run through all the nodes of the BPTAO, and determine, for each node, if there is a seed below it. An edge belongs to the watershed cut if both children of its corresponding node in the BPTAO contain a seed.

Apart from the BPTAO, Algorithm \ref{alg:wsfromseed} requires a binary array \textsf{seeded},  which indicates whether each leaf node $n$ of the tree contains a seed. 
At the start of the algorithm, the array value is \texttt{True} for seeded leaves of the tree, and \texttt{False} elsewhere. 
Algorithm \ref{alg:wsfromseed} then browses the tree from leaves to root to determine which internal nodes have at least one seeded descendant. 
During a second pass, for any internal node $n$ of the tree (in any order), the MST edge corresponding to $n$ is marked as a watershed edge if both children of $n$ contain at least one seed, \emph{i.e.} if several marked regions merge at the node $n$. 
The algorithm returns a watershed cut represented by a binary array mapping each edge of the MST to \texttt{True} if the edge belongs to the cut and \texttt{False} otherwise.

\begin{algorithm2e}[tb]
\small
\caption{watershedFromSeeds}\label{alg:wsfromseed}
\KwData{A BPTAO \texttt{T}}
\KwData{A Boolean array \texttt{seeded} indicating which leaves of \texttt{T} are seeds.}
\KwResult{A Boolean array \texttt{ws} indicating which MST edges belongs to the watershed cut.}
\ForEach{node $n$ of $T$ from leaves (excluded) to root}{
    seeded[n] := seeded[n.c1] \textbf{or} seeded[n.c2]\;
}
\ForEach{non leaf node $n$ of $T$}{
    ws[n.mst\_edge] = seeded[n.c1] \textbf{and} seeded[n.c2]\;
}
\end{algorithm2e}

Fig. \ref{fig:wsfrombpt} shows the result obtained with Algorithm \ref{alg:wsfromseed} on a small example when one vertex of each minimum is seeded (see Section~\ref{sec:selection_seeds}). In this particular case, the cut obtained at the end of the algorithm is what is usually referred to as an (unsupervised) watershed cut of the underlying graph\Cite{cousty-et-al_pami2008}.

\begin{figure}[tb]
    \centering
    \includegraphics[width=.3\textwidth]{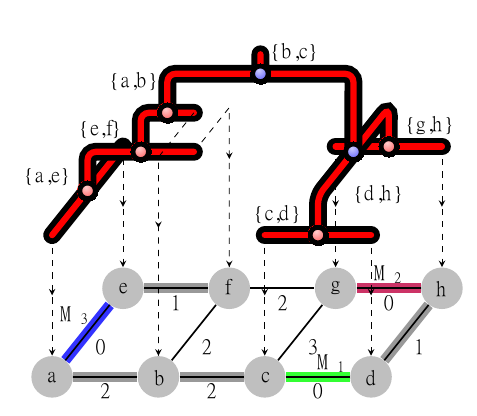}
    \caption{Unsupervised watershed cut. The underlying graph has 3 minima: edges $M_1$ (green), $M_2$ (red) and $M_3$ (blue). By applying Algorithm \ref{alg:wsfromseed} to the BPTAO  where one vertex of each minimum is seeded, we can identify the watershed nodes (blue points) and their associated MST edges $\{b,c\}$ and $\{d,h\}$.}
    \label{fig:wsfrombpt}
\end{figure}

\subsection{Hierarchical watershed cuts from weighted seeds}
\label{sec:hierarchical_watershed}

In the previous subsection, we have seen how watershed cuts can be obtained from a set of arbitrary binary seeds, which generalizes the fully unsupervised watershed cuts rooted in the minima of an input function. 
In the general case described above, all arbitrary seeds have the same ``importance level'', meaning that we are only interested in the final pixels/vertices partition derived from those given seeds. 
We can further generalize seeded watershed cuts by considering that the input seeds are weighted (or ordered) according to a given criterion, which leads to not only a single partition rooted in this set of seeds, but instead to a sequence of partitions rooted in any subset of the $n$ most important seeds ($n$ being less or equal to the total number of seeds). 
As stated in \cite{cousty-et-al_ismm2013}, in the framework of edge-weighted graphs, such sequence of partitions satisfying the causality principle can always be defined through minimum spanning forests rooted in each subset of $n$ most important input seeds.

As introduced in \cite{najman-et-al_ismm2013}, hierarchical watersheds cuts from seeds can be computed in quasi-linear time (given that the input edges are already sorted) thanks to the same framework described in the previous subsection, by post-processing the BPTAO. 
As shown in previous sections, watershed-cut edges can be detected in a single pass on the non-leaf nodes of the tree, resulting in a binary array indicating which MST edges are watershed-cut edges. 
In other words, it provides us with a binary weighted MST from which a watershed-cut segmentation can be obtained by simply computing its connected components. 
Now, given that an ordered set of seeds is provided, we aim at computing not only the set of MST watershed-cut edges, but also their saliency value, which indicates their \emph{level of disappearance} in the hierarchy. 
This can be achieved through a very simple adaptation of Algorithm \ref{alg:wsfromseed}. 

In Algorithm \ref{alg:hwatersheds_seeds}, given a BPTAO and a set of weighted seeds ({\em i.e.}, an array indicating for each leaf node $n$, the weight of the seed associated to $n$ or 0 if $n$ is not seeded), the result is an array $hws$ indicating the saliency values of the MST edges.
The algorithm itself is a direct generalization of Algorithm~\ref{alg:wsfromseed} for binary seeds: as classically done in mathematical morphology, binary operators \texttt{or} and \texttt{and} generalize to real numbers with the operators \texttt{max} and \texttt{min}. 
The first loop of the algorithm then computes the maximum seed weight value among the descendants of each node.
Then, the second loop over internal nodes computes the minimum of the maximum weights of its two children, giving the saliency value of the corresponding MST edge.

Similar to the binary case, the levels of the hierarchical watershed associated to $hws$ can be computed by extracting the connected components of the different levels sets of $hws$ as described in the next section.

\begin{algorithm2e}[tb]
\small
\caption{hierarchicalWatershedFromSeeds}\label{alg:hwatersheds_seeds}
\KwData{A BPTAO \texttt{T}}
\KwData{An array \texttt{w\_seeds} of weights, such that for each leaf \texttt{n} of \texttt{T},\texttt{w\_seeds[n] $\neq$ 0} indicates that \texttt{$n$} is a seed and the positive value \texttt{w\_seeds[n]} provides its weight.
}
\KwResult{An array \texttt{hWS} indicating the saliency of MST edges.}
\ForEach{node $n$ of $T$ from leaves (excluded) to root}{
    w\_seeds[n] := max(w\_seeds[n.c1], w\_seeds[n.c2])\;
}
\ForEach{non leaf node $n$ of $T$}{
    hWS[n.mst\_edge] = min(w\_seeds[n.c1], w\_seeds[n.c2])\;
}
\end{algorithm2e}

\section{Extraction of connected components}
\label{sec:extraction_cc}

In this section, we describe how connected components of (hierarchical) watershed segmentations can be labeled once the watershed-cut edges have been extracted from the BPTAO. We can consider three cases:
\begin{enumerate}
    \item {\em Single segmentation from non-weighted seeds:} given the BPTAO and the watershed-cut edges already computed in the previous step (Section \ref{sec:wat_cut_seeds}), we aim to label the connected components of the resulting watershed segmentation. This can be done by simply employing any connected component labeling algorithm on the graph $G'= (V, E')$, where $E'$ is the set of building edges of the BPTAO that are not watershed-cut edges.
    \item {\em Single segmentation from weighted seeds:} similarly, given the BPTAO and the weighted watershed-cut edges computed in the previous step (Section \ref{sec:hierarchical_watershed}), let us say that we aim to extract a single segmentation level from the associated hierarchical watershed. The target segmentation level may be defined by (a) a threshold $\lambda$ on the saliency value of watershed-cut edges; or (b) the number of regions it contains. In the first case, it suffices to discard the watershed-cut edges whose saliency values are lower than $\lambda$ and proceed with a connected component labeling algorithm on the graph $G'= (V, E')$, where $E'$ is the set of building edges of the BPTAO that are not watershed-cut edges or whose saliency value is lower than $\lambda$. In the second case, if a precise number $r$ of regions is required, we can keep only the $r-1$ watershed-cut edges with the highest saliency values, and proceed in the same manner.
    \item {\em Tree representation of a hierarchical watershed:} in order to keep all hierarchical segmentation levels, we may consider a tree representation of the resulting hierarchy. This can be achieved by computing the BPTAO from the MST $G'=(V,E,w')$ of $G$, whose edge set is the set of building edges of the BPTAO and $w'$ assigns the saliency value to the watershed-cut edges and zero to the remaining ones.
\end{enumerate}

\section{Use cases of the watershed-cut framework}
\label{sec:useCases}
In this section, we present the detailed pipeline for some specific use cases of the framework corresponding to different routes in the global algorithmic scheme presented in Figure~\ref{fig:schema_section3}.
The Python scripts for computing those use cases  are presented in a Python \href{https://www.esiee.fr/~perretb/notebooks/Watersheds_and_hierarchical_watersheds.zip}{Notebook} relying on the \href{https://github.com/higra/Higra}{Higra} -- Hierarchical Graph Analysis -- library.

More specifically, the algorithms described previously for computing BPTAO, watershed-cut edges, (weighted) seeds, and connected components are building blocks for several specific use cases, whose aims range from a single unsupervised watershed-cut segmentation to a tree representation of a seeded hierarchical watershed. 
In the following, we present the pipeline for handling the most important cases. 
We consider that the input is an RGB or gray-scale image whose contours are extracted using any given contour detection tool, and whose edge-weighted graph is computed directly from the extracted contours. 
In the demonstration Notebook, contours are extracted using the Structured Edge Detection \cite{dollar2013structured}.

\subsection{Unsupervised watershed-cut}

We start with the completely unsupervised (non-hierarchical) watershed-cut seeded in the set of minima of the input graph, as described in Section \ref{sec:wsc}.  
Given an edge-weighted graph, the method relies on the following  main steps, illustrated in Figure \ref{fig:use_case_unsupervised_watershed_cut}:
\begin{enumerate}
    \item Construction of the BPTAO and associated MST (Section~\ref{sec:bpt}, Algorithm~\ref{alg:bpt});
    \item Identification of the minima seeds  and their weights (Section~\ref{sec:seedminima}, Algorithms~\ref{alg:minima_ann} and~\ref{alg:minima_one_leaf_ann});
    \item Identification of watershed-cut nodes (Section~\ref{sec:wat_cut_seeds}, Algorithm~\ref{alg:wsfromseed});
    \item Connected component labeling of the watershed-cut in the MST (Section~\ref{sec:extraction_cc}).
\end{enumerate}
Figure~\ref{fig:use_case_unsupervised_watershed_cut_demo} demonstrates an application of this pipeline.

\begin{figure*}[htbp]
    \centering
    \includegraphics[width=\textwidth]{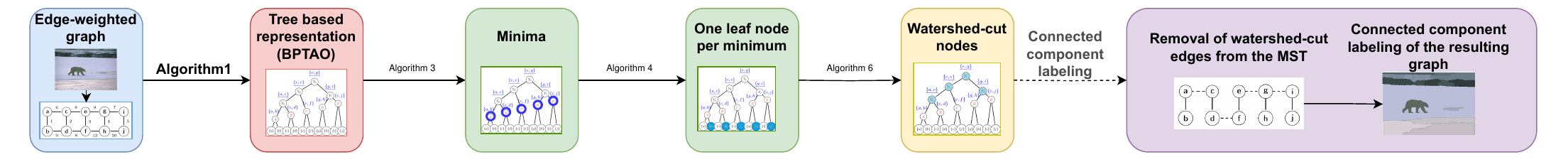}
    \caption{Unsupervised watershed-cut pipeline.}%
    \label{fig:use_case_unsupervised_watershed_cut}
\end{figure*}

\begin{figure*}[ht]
    \centering
\includegraphics[width=0.25\textwidth]{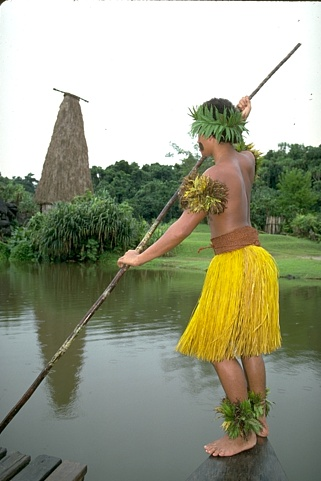}
\includegraphics[width=0.25\textwidth]{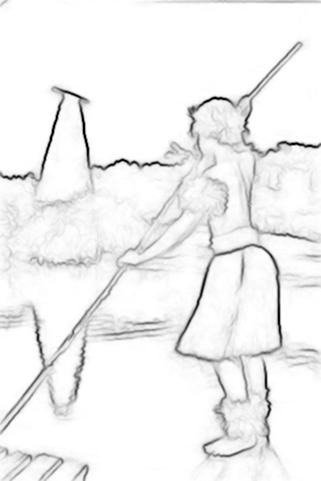}
\includegraphics[width=0.25\textwidth]{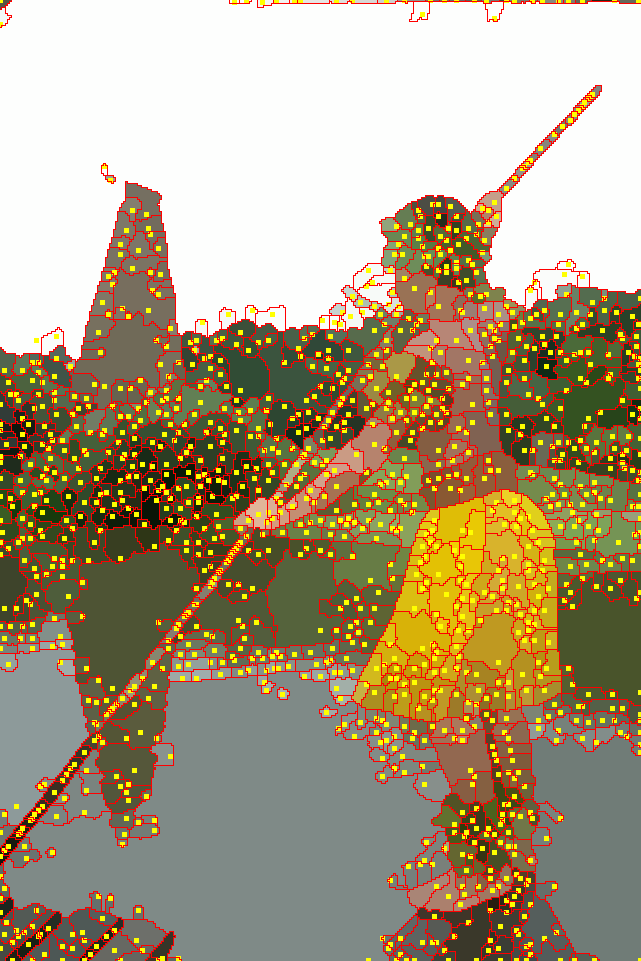}
    \caption{Demonstration of an unsupervised watershed-cut pipeline. From left to right, we have the original image (source BSDS500~\cite{amfm_pami2011}), its gradient (computed using SED~\cite{dollar2013structured}), and its unsupervised watershed cut. In the cut, each region is filled with its average color, and the contours are drawn in red. Each region contains a single yellow point depicting the location of its associated minimum.}
    \label{fig:use_case_unsupervised_watershed_cut_demo}
\end{figure*}

\subsection{Seeded watershed cut}
\label{sec:usecase_swc}

In this scenario, we consider a (non-hierarchical) watershed-cut seeded on a set of seeds which are not necessarily the regional minima of the input graph. 
As discussed in Section~\ref{sec:selection_seeds}, we may consider several approaches for extracting seeds to handle specific applications ({\em e.g.} manual, or deep learning methods). 
Let us say that we are given an edge-weighted graph and a set of seeds represented as a set of vertices. 
We then proceed as follows (see Figure \ref{fig:use_case_seeded_watershed_cut}):

\begin{enumerate}
    \item Construction of the BPTAO and associated MST (Section~\ref{sec:bpt}, Algorithm~\ref{alg:bpt});
    \item Identification of watershed-cut edges (Section~\ref{sec:wat_cut_seeds}, Algorithm~\ref{alg:wsfromseed});
    \item Connected component labeling of the watershed-cut in the MST (Section~\ref{sec:extraction_cc}).
\end{enumerate}

This case is indeed a direct generalization of the previous one, where seeds are considered as an input.
Figure~\ref{fig:use_case_supervised_watershed_cut_demo} demonstrates an application of this pipeline.

\begin{figure*}[htbp]
    \centering
    \includegraphics[width=\textwidth]{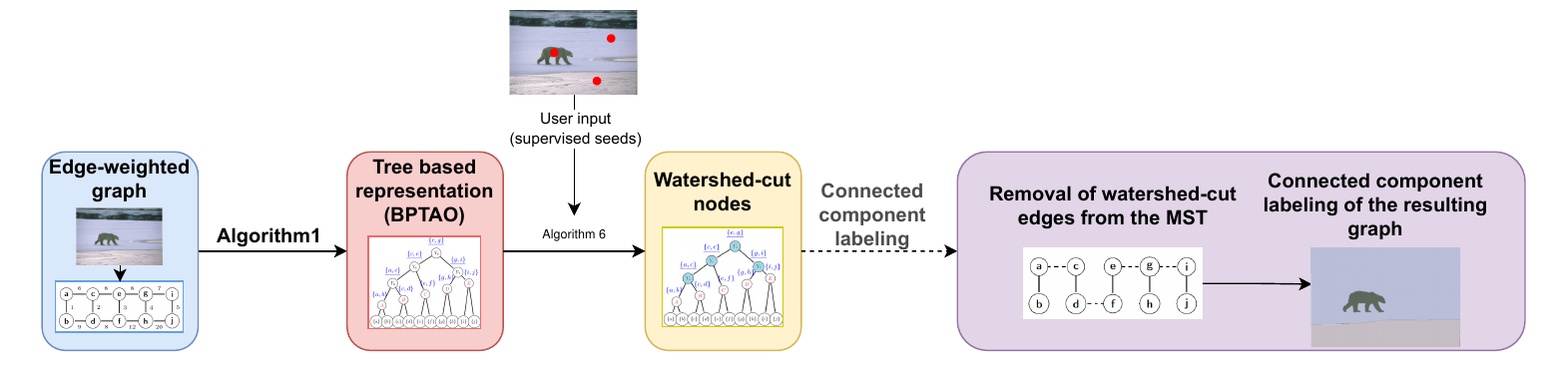}
    \caption{Seeded watershed-cut pipeline.}
    \label{fig:use_case_seeded_watershed_cut}
\end{figure*}

\begin{figure}[ht]
    \centering
\includegraphics[width=0.23\textwidth]{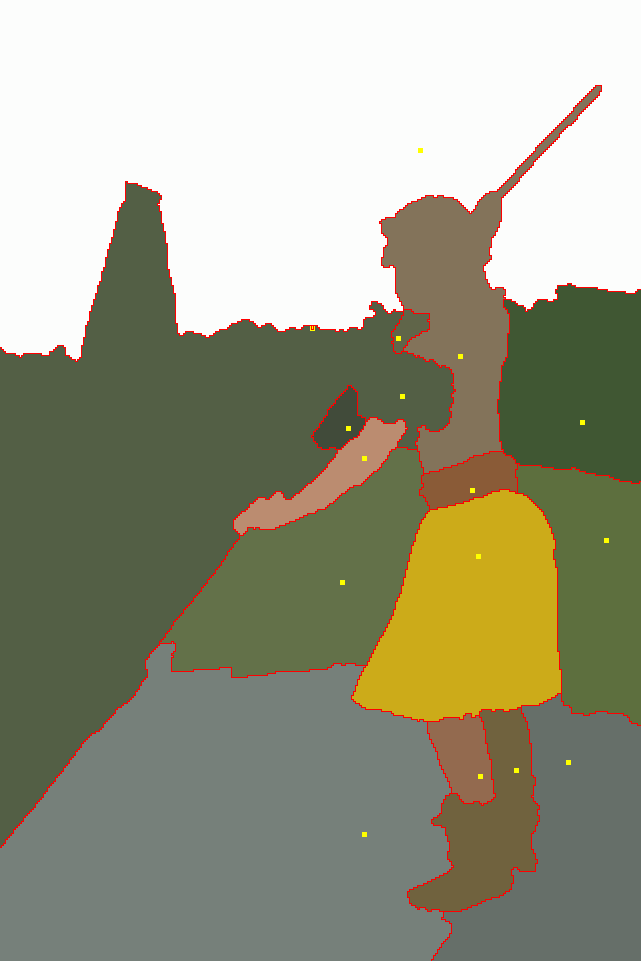}
\includegraphics[width=0.23\textwidth]{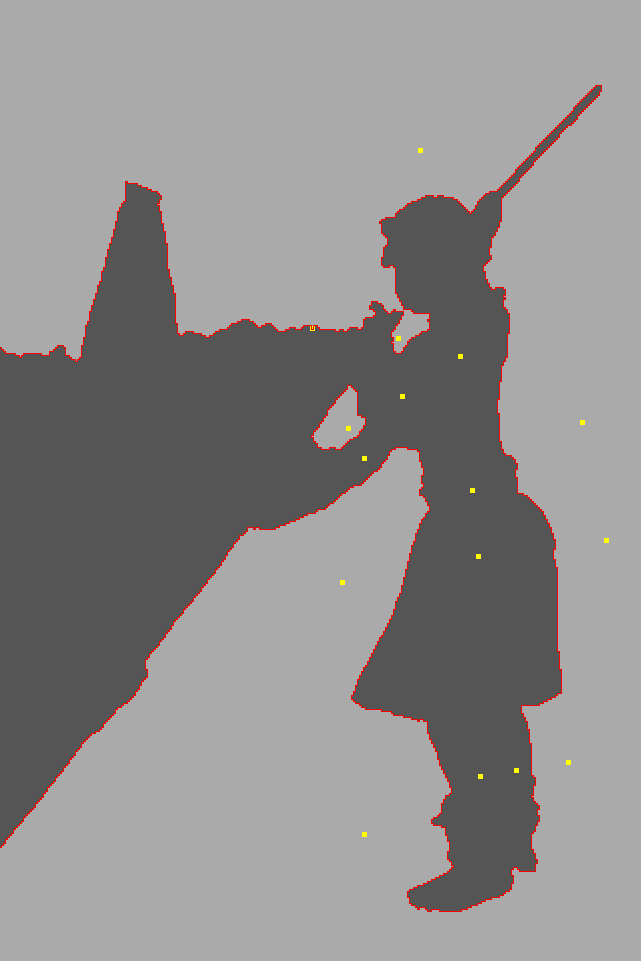}
    \caption{The image and gradient considered are the same as in Figure~\ref{fig:use_case_unsupervised_watershed_cut_demo}. The user provides two sets of pixels as seeds: a set of foreground seeds and a set of background seeds. On the left is the supervised watershed cut. Each region is filled with its average color, and the contours are drawn in red. Each region contains a single yellow point depicting the location of its associated seed. Right: the final segmentation obtained by grouping the regions associated with the foreground and background seeds. It can be seen that the contours are well placed overall with a small number of seeds; however, a gradient leak on the person's shoulder causes a large error.}
    \label{fig:use_case_supervised_watershed_cut_demo}
\end{figure}

\subsection{Tree representation of an area-based hierarchical watershed}

\label{subsection:user_case_tree_area}

Let us now consider a tree representation of a hierarchical watershed by area as our final goal, as illustrated in Figure \ref{fig:use_case_tree_representation_watershed}. 
Similarly to the first scenario, the seeds are the regional minima of the input graph. However, we will now weigh the seeds with area-based extinction values of their associated minima:

\begin{enumerate}
    \item Construction of the BPTAO and associated MST (Section~\ref{sec:bpt}, Algorithm~\ref{alg:bpt});
    \item Computation of area-based extinction values on the nodes of BPTAO (Section \ref{subsection:seeds_hierarchical}, Algorithm \ref{alg:ev});
    \item Identification of the minima seeds  and their weights (Section~\ref{sec:seedminima}, Algorithms~\ref{alg:minima_ann} and~\ref{alg:minima_one_leaf_ann});
    \item Computation of the saliency values of watershed-edges in the MST (Section \ref{sec:hierarchical_watershed}, Algorithm \ref{alg:hwatersheds_seeds});
    \item Construction of the BPTAO of the MST weighted by the saliency values  (Section~\ref{sec:bpt}, Algorithm~\ref{alg:bpt}).
\end{enumerate}

The reader may note that the BPTAO algorithm (Algorithm \ref{alg:bpt}) is employed twice, once on the input  weighted graph $G$, and, in the end, on the reevaluated MST of $G$, which results in the tree representation of the area-based hierarchical watershed of $G$. 

The left subfigure in Figure~\ref{fig:use_case_hierarchical_watershed_demo} demonstrates an application of this pipeline.

\begin{figure*}[htbp]
    \centering
    \includegraphics[width=\textwidth]{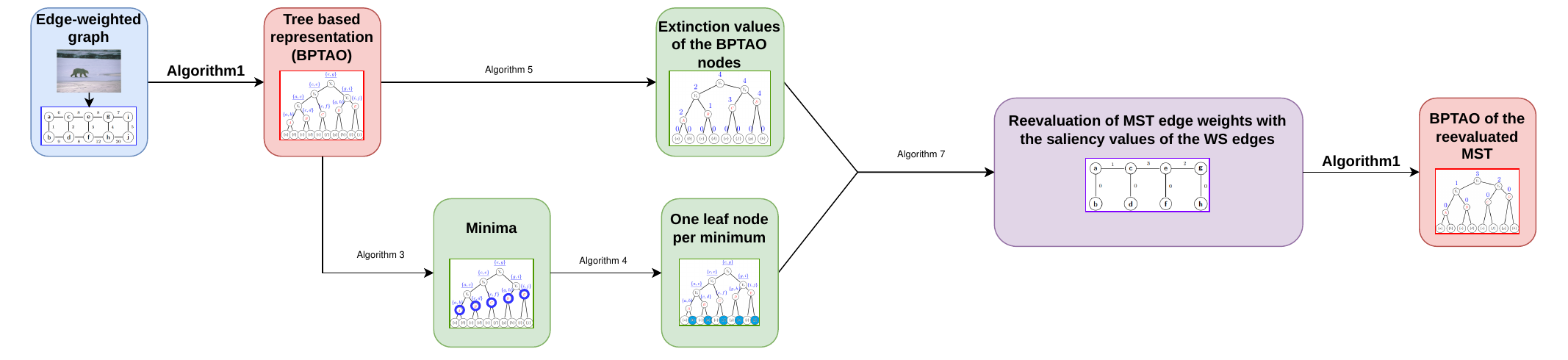}
    \caption{Pipeline for computing a tree representation of a hierarchical watershed.}
    \label{fig:use_case_tree_representation_watershed}
\end{figure*}

\subsection{Watershed cut from a hierarchical watershed}

Once a hierarchical watershed is computed, there are several ways of exploring this representation. 
Given that the input graph has $n$ regional minima, we may want to extract a segmentation level composed of the $k$ most important regions according to the attribute used during the hierarchy construction. 
In this case, we can reuse the first four steps described in the previous section. However, instead of computing the tree representation of the reweighted MST, we apply a threshold to obtain a cut, and then labelise the result.
\begin{itemize}
    \item [5.] Threshold the reevaluated MST weights by keeping only the watershed-cut edges with the $k$ largest weights; and
    \item [6.] Connected component labeling of the watershed cut in the MST (Section~\ref{sec:extraction_cc}).
\end{itemize}
Alternatively, we can give a specific threshold value on the watershed-cut nodes instead of the desired number of regions.

The right subfigure in Figure~\ref{fig:use_case_hierarchical_watershed_demo} demonstrates an application of this pipeline.

\begin{figure}[ht]
    \centering
\includegraphics[width=0.23\textwidth]{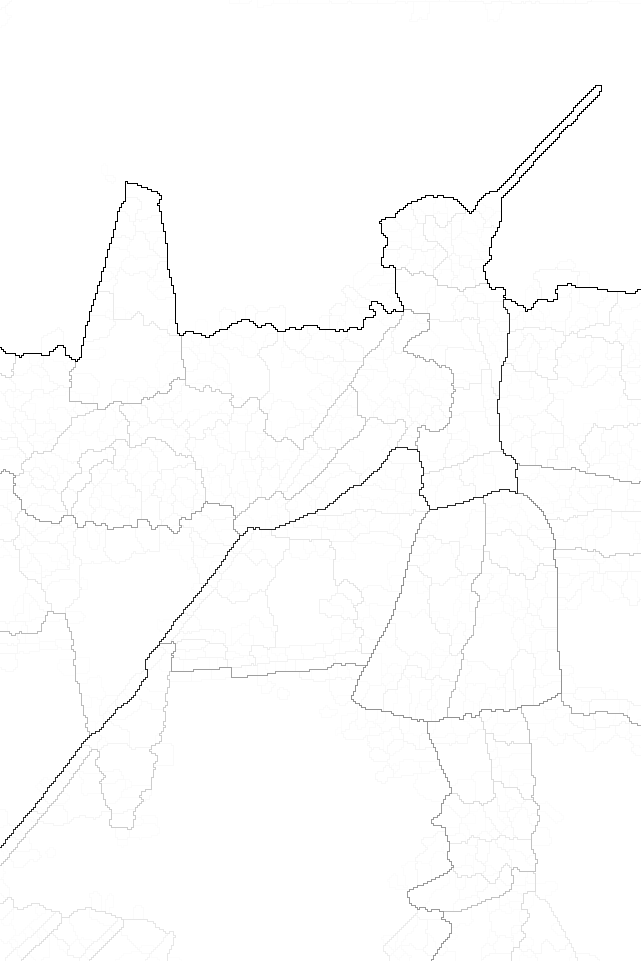}
\includegraphics[width=0.23\textwidth]{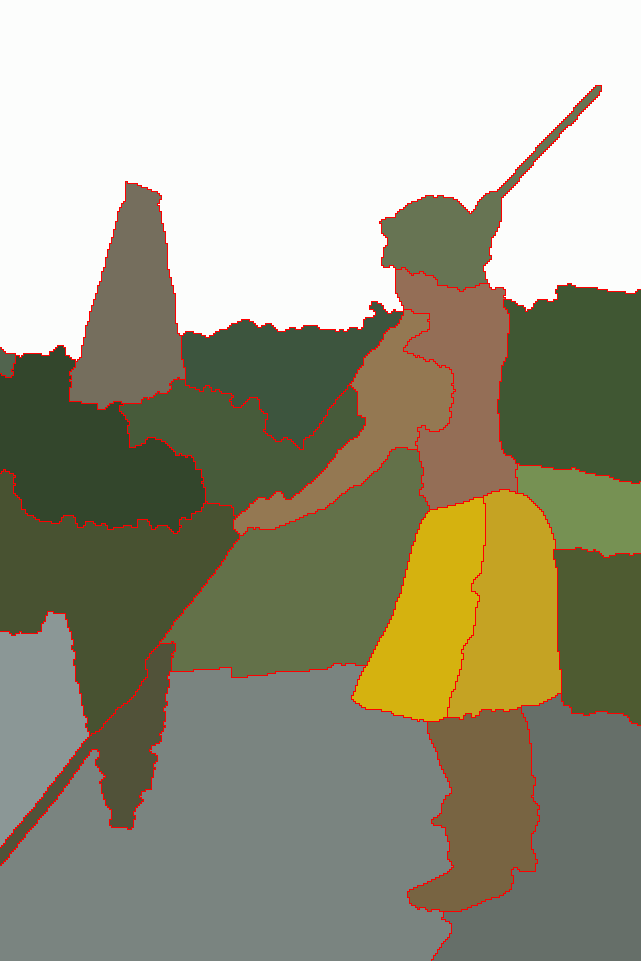}
    \caption{Demonstration of tree-based representation and watershed cut from a hierarchical watershed. The image and gradient considered here are the same as in Figure~\ref{fig:use_case_unsupervised_watershed_cut_demo}. On the left is a saliency map of a hierarchical watershed by area. Each upper threshold of the saliency map represents a segmentation of the image. The right-hand image shows the watershed cut obtained by thresholding the saliency map at a level containing 20 regions.}
    \label{fig:use_case_hierarchical_watershed_demo}
\end{figure}

\subsection{Seeded hierarchical watershed cut}

In the previous use case, we assumed that the desired segmentation can be found by a single threshold on the saliency map. 
As this is not always the case, we can consider a seed based segmentation extraction, in which user input seeds $S$ are used {\em after} a hierarchical watershed has been obtained, as explained in Section \ref{sec:seeded_hierarchical_watershed}. Given the first four steps described in Section \ref{subsection:user_case_tree_area}, we proceed as follows:
\begin{itemize}
    \item [5.] Construction of the BPTAO of the MST weighted by the saliency values   obtained in step~4 of Section \ref{subsection:user_case_tree_area};
    \item [6.] Identification of watershed-cut edges (Section~\ref{sec:wat_cut_seeds}, Algorithm~\ref{alg:wsfromseed});
    \item [7.] Connected component labeling of the watershed-cut in the MST (Section~\ref{sec:extraction_cc}).  
\end{itemize}
Once the hierarchical watershed has been computed, we can see that the method is exactly the same as the seeded watershed method presented in Section~\ref{sec:usecase_swc}, but applied to the BPTAO of the reweighted MST rather than the original graph.

Figure~\ref{fig:seeded_hierarchical_watershed_cut_demo} demonstrates an application of this pipeline.

\begin{figure}[ht]
    \centering
\includegraphics[width=0.23\textwidth]{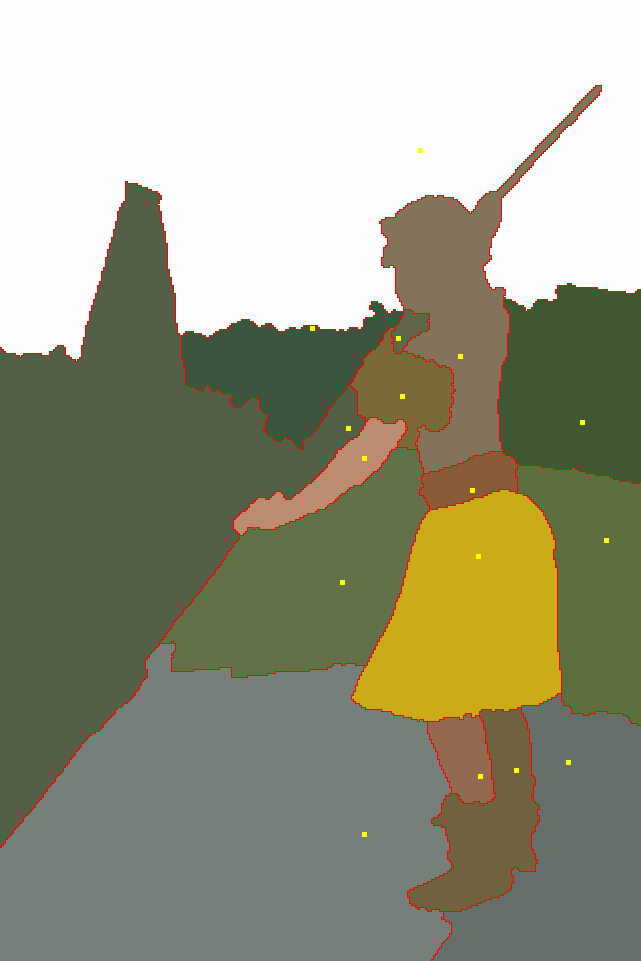}
\includegraphics[width=0.23\textwidth]{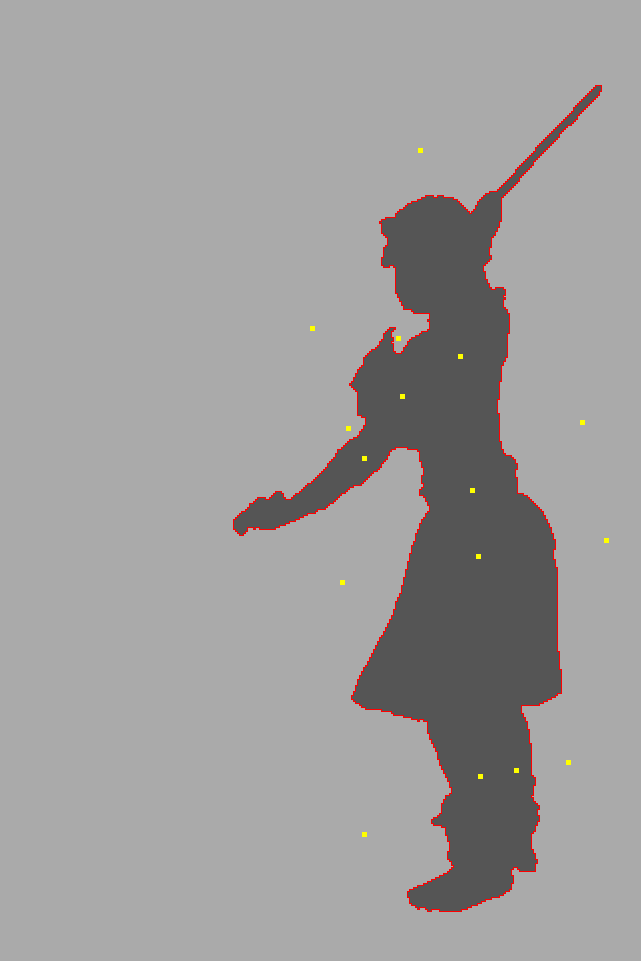}
    \caption{Demonstration of seeded hierarchical watershed-cut pipeline. The image and gradient considered here are the same as in Figure~\ref{fig:use_case_unsupervised_watershed_cut_demo}. The user provides two sets of seeds: a set of foreground seeds and a set of background seeds. On the left is a seeded hierarchical watershed cut based on area. Each region is filled with its average color, and the contours are drawn in red. Each region contains a single yellow point depicting the location of its associated seed. Right: the final segmentation obtained by grouping regions associated with the foreground and background seeds. Compared to Figure~\ref{fig:use_case_supervised_watershed_cut_demo}, the area-based hierarchical watershed helped reduce leakage in the raw gradient, leading to correct boundary placement around the person's shoulder.}
    \label{fig:seeded_hierarchical_watershed_cut_demo}
\end{figure}

\section{Conclusion}

In this article, we have given a comprehensive review of a general framework for computing (hierarchical) watersheds in edge-weighted graphs. This framework revolves around three main ideas: 1) the binary partition tree by altitude ordering encodes the topology of the minima of the graph, 2) a subset of the graph vertices called seeds induces a set of watershed nodes in the tree and 3) those watershed nodes are in bijection with the watershed cut of the graph induced by the seeds. 
Moreover, if the seeds are weighted (for example according to an attribute value), it is then possible to associate a persistence value to the induced watershed nodes which then translates as a saliency map, \emph{i.e.}, a weighted graph-cut or a hierarchy, in the graph space. 
We provide simple and efficient algorithms for each step of this framework, and we demonstrate the modularity of the framework with several use cases. 
All the algorithms and use cases are demonstrated in the Python \href{https://www.esiee.fr/~perretb/notebooks/Watersheds_and_hierarchical_watersheds.zip}{Notebook} relying on the \href{https://github.com/higra/Higra}{Higra} -- Hierarchical Graph Analysis -- library.

The usages and challenges around watersheds and especially watershed hierarchies of course go beyond the use cases presented in the article. We can cite among others: segmentation hierarchy filtering\Cite{soille2008constrained,perret2019removing}, energy based optimization and non horizontal-cuts\Cite{guigues-et-al_ijcv2006,xu2016hierarchical} and combination with machine learning techniques\Cite{chierchia2019ultrametric,Tripletws}.

In future work we plan to adapt and develop new parallel, distributed and massively parallel algorithms (targeted for GPU), building on our existing works on out-of-core (hierarchical) watersheds\Cite{lefevre2025out}.

\section*{Declarations}

\begin{itemize}
\item \textbf{Funding}: The work of L. Najman was partially supported by
  Khalifa University of Science and Technology under Award No. KU-INT-FSU-2026-004-8471000027.
\item \textbf{Conflict of interest}: The authors declare no conflict of interest.
\item \textbf{Data availability}: No datasets were generated or analysed during the current study.
\item \textbf{Author contributions}: All authors contributed equally to the writing and revision of the manuscript.
\end{itemize}

\bibliography{Bibliography}

\appendix
\section{Appendix: Proof of Theorem~\ref{thm:multi-scaleOptimality}}
\label{sec:ProofTheo4}

Let us start this section with two results which can be
straightforwardly derived from classical results on trees and on
minimum spanning trees. To deduce these lemmas, the construction
presented in Sec. III-B of \cite{cousty-et-al_pami2008} linking
minimum spanning trees and forests can be considered.

\begin{lemma}
\label{lem:1}
Let~$\mathcal{M}_1$ and~$\mathcal{M}_2$ be two sets of minima of~$G$
such that~$\mathcal{M}_2 \subseteq \mathcal{M}_1$, then any minimum
spanning forest relative to~$\mathcal{M}_1$ is included in a minimum
spanning forest relative to~$\mathcal{M}_2$.
\end{lemma}

Let~$X$ be a subgraph of~$(V,E)$. We denote by~$V(X)$ (resp.~$E(X)$) the set of vertices (resp. edges) of~$X$. Let~$u \in E(X)$.We write~$X \setminus u$ for the graph $(V(X), E(X) \setminus \{u\})$. 
If~$v=\{x,y\} \in E \setminus E(X)$, we write~$X \cup v$ for the graph~$(V(X) \cup \{x,y\}, E(X) \cup \{v\})$.

\begin{lemma}[lemma 17 in \cite{cousty-et-al_pami2008}]
\label{lem:2}
Let~$\mathcal{M}$ be a marker set and let~$Y$ be a spanning forest rooted in~$\mathcal{M}$. If, for any~$u \in E(Y)$ and any~$v \in E$ such that $(Y
\setminus u) \cup v$ is a spanning forest rooted in~$\mathcal{M}$, we
have $w(u) \leq w(v)$, then~$Y$ is an MSF rooted
in~$\mathcal{M}$.
\end{lemma}

Let us state a first result on hierarchical watersheds obtained from one-by-one filtrations. This result is a consequence of the fact that the attribute used to select the connected components which are kept in the resulting map is increasing and that the filtration series is one-by-one. This result is the key ingredient in the subsequent proof of Theorem~\ref{thm:multi-scaleOptimality}

Let~$\Phi = (\Phi_0(G)=G, \ldots, \Phi_N(G))$ be a filtration series of~$G$. We say that~$\Phi$ is {\em one-by-one} if, for any~$\lambda \in \{1, \ldots, N\}$,  we have~$n_\lambda = n_{\lambda-1}-1$, where~$n_\lambda$ and $n_{\lambda-1}$ are the numbers of minima of~$\Phi_\lambda(G)$ and~$\Phi_{\lambda-1}(G)$, respectively.

\begin{lemma}
  \label{lem:3} Let~$\Phi = (\Phi_0(G)=G, \ldots, \Phi_N(G))$ be a one-by-one filtration series of~$G$. Let~$\mathcal{H} = (C_0, \ldots, C_N)$ be a hierarchical watershed with respect to~$\Phi$. Let~$i \in \{0, \ldots, N\}$.
  Then, for any edge~$u$ in the cut set induced by~$C_i$, for any~$j \in \{0, \ldots, i\}$, we have~$\Phi_i(G)(u) = \Phi_j(G)(u)$.
\end{lemma}
\begin{proof} In this proof, to simplify the notations, we write~$\Phi_i$ instead of~$\Phi_i(G)$ when no confusion may occur.  By abuse of notation, we also write~$\Phi_i$ to denote the weight map of the edge-weighted graph~$\Phi_i(G)$. The proof is done by induction.

(1)  {\bf Induction base}. If~$i = 0$, the property is trivial.

(2) {\bf Induction hypothesis}. Let~$i \geq 1$. Assume that for any~$\lambda \in \{0, \ldots, i-1\}$,  for any edge~$u$ in the cut set induced by~$C_\lambda$, and for any~$j \in \{0, \ldots, \lambda\}$, we have~$\Phi_\lambda(G)(u) = \Phi_j(G)(u)$.

(3)  {\bf Induction}. In order to complete the proof of
    Lemma~\ref{lem:3} by induction, we
    are going to show that if the induction hypothesis holds true, then,  for any edge~$u = \{x,y\}$ in the cut set induced by~$C_i$, and for any~$j \in \{0, \ldots, i\}$, we have~$\Phi_i(u) = \Phi_j(u)$. Let~$u = \{x,y\}$ be any edge in the cut set of~$C_i$. 
    Let~$R_x$ (resp.~$R_y$) be the region of~$C_i$ that contains~$x$ (resp.~$y$). As~$C_i$ is a watershed cut of~$\Phi_i$, there exists a unique minimum~$M_x$ (resp.~$M_y$) of~$\Phi_i$ that is included in~$R_x$ (resp.~$R_y$). 
    We denote by~$C'$ the connected component in~$C(G,\Phi_{i-1}, \Phi_{i-1}(u))$ that contains~$u$. Let~$R'_x$ (resp.~$R'_y$) be the region of~$C_{i-1}$ that contains~$x$ (resp.~$y$). 
    As~$C_{i-1}$ is a watershed cut of~$\Phi_{i-1}$, (i) there exists a unique minimum~$M'_x$ (resp.~$M'_y$) of~$\Phi_{i-1}$ that is included in~$R'_x$ (resp.~$R'_y$), and (ii) the minima $M'_x$ and~$M'_y$ are distinct. 
    Furthermore, by the drop of water principle, there exists a descending path from~$x$ (resp.~$y$) to~$M'_x$ (resp. $M'_y$). Thus,~$M'_x$ (resp.~$M'_y$) is included in~$C'$. 
    As~$\mathcal{H}$ is a hierarchy, we have $R'_x \subseteq R_x$ (resp. $R'_y \subseteq R_y$). Hence, we deduce that~$M'_x \subseteq R_x$ (resp.~$M'_y \subseteq R_y$). 
   
   As~$\Phi$ is one-by-one, at least one of the following statements holds true: (i) $M'_x \subseteq M_x$ or~(ii) $M'_y \subseteq M_y$, otherwise the number of minima of~$\Phi_i$ would decrease by at least two with respect to the one of~$\Phi_{i-1}$. This implies that we either have~$M_x \cap C' \neq \emptyset$ or~$M_y \cap C' \neq \emptyset$. The graphs~$M_x, M_y,$ and $C'$ are connected components of~$\Phi_{i-1}$ (more precisely for each of them, there exists~$k$ such that it belongs to~$C(G, \Phi_{i-1},k)$). It is well known that any two connected components of a weighted graph are either nested or disjoint. Thus, we have~$M_x \subseteq C'$ or~$M_y \subseteq C'$. The minima~$M_x$ and~$M_y$ are two connected components of~$\Phi_i$. Thus,~the measure~$\mu$ of~$M_x$ and~$M_y$ involved in the filtration is greater than the scale parameter~$\lambda_i$ associated to~$\Phi_i$. As~$\mu$ is increasing, we deduce that the measure~$\mu$ of the region~$C'$ is also  greater than the scale parameter~$\lambda_i$. 
    Thus,~$C'$ is preserved as a connected component of~$\Phi_i$ at the same level as in $\Phi_{i-1}$ (\ie~$C' \in C(G, \Phi_i, \Phi_{i-1}(u))$).
As~$u$ belongs to~$C'$, we deduce that~$\Phi_{i}(u) = \Phi_{i-1}(u)$. As~$\mathcal{H}$ is a hierarchy, since~$u$ belongs to the cut set of~$C_i$, it also belongs to the cut set of~$C_{i-1}$. Thus, from the induction hypothesis, we deduce that $\Phi_{i}(u) = \Phi_{j}(u)$, for any~$j \in \{0, \ldots, i\}$, which completes the proof.
    \end{proof}

We first prove the forward implication of Theorem~\ref{thm:multi-scaleOptimality}. It is a straightforward consequence of Theorem~\ref{thm:seededWC} or, more precisely, of the following theorem established by Allène {\em et al}.

\begin{thm}(from Theorem 5.2 in \cite{allene-et-al_ivc2010})
Let~$\mathcal{M}$ be a marker set and let~$(\Phi_0, \Phi_1)$ be a  filtration series. Any minimum spanning forest rooted in~$\mathcal{M}$ for~$\Phi_0$ is a minimum spanning forest rooted in~$\mathcal{M}$ for~$\Phi_1$.
\end{thm}
{
\begin{proof}[Proof of the forward implication of Theorem~\ref{thm:multi-scaleOptimality}]
Let~$\lambda \in \{0,\ldots,N\}$, and let us denote by~$\mathcal{M}_\lambda$ the set which contains every minimum~$M$ of~$G$ whose extinction value (for the filtration series~$\Phi$) is greater than or equal to~$\lambda$ (\ie{} $\epsilon(M) \geq \lambda$). Assume that condition~(i) of Theorem~\ref{thm:multi-scaleOptimality} holds true. Then~$F_\lambda$ is a minimum spanning forest of~$\Phi_0(G)=G$ rooted in~$\mathcal{M}_\lambda$. By applying the previous theorem along the filtration from~$\Phi_0(G)$ to~$\Phi_\lambda(G)$, we obtain that~$F_\lambda$ is also a minimum spanning forest of~$\Phi_\lambda(G)$ rooted in~$\mathcal{M}_\lambda$. Since~$\mathcal{M}_\lambda$ is the set of regional minima of~$\Phi_\lambda(G)$, Theorem~\ref{thm:WCOptimality} implies that~$C_\lambda$ is a watershed cut of~$\Phi_\lambda(G)$. This holds for every~$\lambda$, and~$\mathcal{H}$ satisfies the hierarchy principle by assumption. Therefore~$\mathcal{H}$ is a hierarchical watershed with respect to~$\Phi$.
\end{proof}}

We will now focus on the backward implication of Theorem~\ref{thm:multi-scaleOptimality}, namely on the following property.

\begin{proper}[hierarchical watershed optimality]
\label{prop:OptimalityOfHierarchicalWS}
 Let~$\Phi = (\Phi_0(G)=G, \ldots, \Phi_N(G))$ be a one-by-one filtration series of~$G$. Let~$\mathcal{H} = (C_0, \ldots, C_N)$ be a hierarchical watershed with respect to~$\Phi$. Then, for any value~$\lambda$ in~$\{0, \ldots, N\}$, there exists a minimum spanning forest~$F_\lambda$ of $G$ which is rooted in the set that contains every minimum of~$G$ whose extinction value is greater than { or equal to}~$\lambda$ and whose connected component partition is~$C_\lambda$.
\end{proper}

\begin{proof}
In order to simplify the writing of the proof, for any~$i \in \{0, \ldots, N\}$, let us denote by~$\mathcal{M}_i$ the set which contains every minimum~$M$ of~$G$ whose extinction value (for the filtration series~$\Phi$) is greater than or equal to~$i$ (\ie{} $\epsilon(M) \geq i$). In this proof, when no confusion may occur, we also write~$\Phi_i$ instead of~$\Phi_i(G)$. Furthermore, by abuse of notation, we write~$\Phi_i$ for the weight map of the edge-weighted graph~$\Phi_i(G)$. The proof of Property~\ref{prop:OptimalityOfHierarchicalWS} is made by induction.
    
    (1) {\bf Induction base}. As~$\mathcal{M}_0$ is the set which contains every minimum of~$G$, by Theorem~\ref{thm:WCOptimality}, we can affirm that there exists a minimum spanning forest~$F_0$ of $G$ which is rooted in~$\mathcal{M}_0$ and whose connected component partition is~$C_0$.
    
    (2) {\bf Induction hypothesis}. Let~$i \geq 1$. Assume that for any~$\lambda \in \{0, \ldots, i-1\}$, there exists a minimum spanning forest~$F_\lambda$ of~$G$ which is rooted in~$\mathcal{M}_\lambda$ and whose connected component partition is $C_\lambda$.
    
    (3) {\bf Induction}. In order to complete the proof of
    Property~\ref{prop:OptimalityOfHierarchicalWS} by induction, we
    are going to show that if the induction hypothesis holds true, then,
    there also exists a minimum spanning forest~$F_i$ of~$G$ which is
    rooted in~$\mathcal{M}_i$ and whose connected component partition
    is $C_i$. By induction hypothesis, there exists a minimum spanning
    forest~$F_{i-1}$ of~$G$ which is rooted in~$\mathcal{M}_{i-1}$ and
    whose connected component partition is~$C_{i-1}$.  Since~$\Phi$ is a
    {one-by-one filtration, there exists a unique minimum~$M^\star \in \mathcal{M}_{i-1}\setminus \mathcal{M}_i$.}
    Hence, as~$C_{i-1}$ is a watershed cut  of~$\Phi_{i-1}$, there exists a region~$A_1$ of~$C_{i-1}$ which
    {includes~$M^\star$.} 
    {Since~$\mathcal{H}$ is a hierarchy, there exists a region~$B$ in~$C_i$ that includes~$A_1$. Again, since~$\mathcal{H}$ is a hierarchy,~$B$ is a union of regions of~$C_{i-1}$. Since~$\Phi$ is one-by-one, the number of regions decreases by exactly one from~$C_{i-1}$ to~$C_i$. Therefore, there exists a unique region~$A_2$ of~$C_{i-1}$ such that~$B = A_1 \cup A_2$.}
    Let~$u$ be an edge of minimum weight (for the weight map~$\Phi_0$) in the set that contains every edge of~$G$
    which contains a vertex in~$A_1$ and a vertex in~$A_2$.
    {The existence of~$u$ follows from the fact that~$B$ is connected. We are going to establish that~$F_{i-1} \cup \{u\}$ is an MSF of~$G$ rooted in~$\mathcal{M}_i$ whose connected component partition is~$C_i$, which will complete the proof of Property~\ref{prop:OptimalityOfHierarchicalWS}. The connected component partition of~$F_{i-1} \cup \{u\}$ is~$C_i$, since adding~$u$ only "merges" the two components~$A_1$ and~$A_2$ of~$F_{i-1}$. Moreover, it can be observed that~$F_{i-1} \cup \{u\}$ is rooted in~$\mathcal{M}_i$ as the only minimum removed from~$\mathcal{M}_{i-1}$ to obtain~$\mathcal{M}_i$ is the one included in~$A_1$, while~$A_2$ includes the unique minimum of~$\Phi_i$ included in~$B$. By Lemma~\ref{lem:1}, as~$\mathcal{M}_i \subseteq \mathcal{M}_{i-1}$, the forest~$F_{i-1}$ is included in at least one MSF rooted in~$\mathcal{M}_i$. Any such forest is obtained from~$F_{i-1}$ by adding a single edge outgoing from~$A_1$. Hence, by Lemma~\ref{lem:2}, in order to establish that~$F_{i-1} \cup \{u\}$ is an MSF of~$G$ rooted in~$\mathcal{M}_i$, it is sufficient to prove that the weight for~$G$ of any edge~$v$ outgoing from~$A_1$ is greater than or equal to the weight of~$u$.}
    Let~$v$ be any edge outgoing from~$A_1$. We distinguish two cases. First,
    if~$v$ contains a vertex in~$A_2$, by definition of~$u$, we
    have~$\Phi_0(v) \geq \Phi_0(u)$. Secondly, let us now assume
    that~$v$ does not contain a vertex in~$A_2$. Thus, the edge~$v$ is
    outgoing from~$B$ and therefore belongs to the cut set
    of~$C_i$. Let~$x$ and~$y$ be such that~$v = \{x,y\}$ and~$x \in
    A_1$. Since~$\Phi$ is one-by-one and 
    {$M^\star \subseteq A_1$} 
    is not included in any minimum of~$\Phi_i$, the unique minimum~$M_x$ of~$\Phi_i$ included in~$B$ is included in~$A_2$. 
    { As~$C_i$ is a watershed cut of~$\Phi_i$, and since~$v$ belongs to the cut set of~$C_i$,}
    there exists a path~$\pi_x =(x_0=x, \ldots,~x_k)$ which is descending for~$\Phi_i$, which does not cross the cut set of~$C_i$, and which is such that~$x_k$ belongs to~$M_x$.
    {Thus, the path~$\pi_x$ remains inside~$B$. Moreover, by definition of a watershed cut, the weight of any edge~$v'$ appearing in~$\pi_x$ satisfies~$\Phi_i(v') \leq \Phi_i(v)$. Thus, since~$x \in A_1$ and~$x_k \in A_2$, there exists an edge~$u'=\{x_j,x_{j+1}\}$ along~$\pi_x$ such that~$x_j \in A_1$ and~$x_{j+1}\in A_2$.}
    By definition of~$u$, we have
    \begin{equation}
      \label{eqn:prJC1}
        \Phi_0(u') \geq \Phi_0(u)
    \end{equation}
     As previously observed, we also have:
    \begin{equation}
      \label{eqn:prJC2}
      \Phi_i(v) \geq \Phi_i(u').
    \end{equation}
    Moreover, by definition of the filtration sequence (made from
    extensive connected operators), we also have:
    \begin{equation}
            \label{eqn:prJC3}
        \Phi_i(u') \geq \Phi_0(u').
    \end{equation}
    As~$v$ belongs to the cut set induced by~$C_i$, by Lemma~\ref{lem:3}, we deduce that~$\Phi_i(v) = \Phi_0(v)$. Thus, from Equation~\ref{eqn:prJC2}, we deduce  that~$\Phi_0(v) \geq \Phi_i(u')$. Thus, from
    Equation~\ref{eqn:prJC3}, we also have~$\Phi_0(v) \geq \Phi_0(u')$, which, from  Equation~\ref{eqn:prJC1} implies that~$\Phi_0(v) \geq \Phi_0(u)$, which completes the proof of Property~\ref{prop:OptimalityOfHierarchicalWS}.
\end{proof}

\end{document}